\documentclass[10pt]{article} %
\usepackage[english]{babel}
\usepackage[round,comma]{natbib}
\ifx\ulesred\undefined
\usepackage{color}

\definecolor{ulesred}{rgb}{0.65,0.18,0.21}

\newcommand{\ulesred}{\color{ulesred}}

\definecolor{darkgreen}{rgb}{0,0.5,0.2}

\definecolor{verylightgray}{gray}{0.85} 
 
\definecolor{mediumgray}{gray}{0.5} 
 
\definecolor{darkgray}{gray}{0.4}

\definecolor{lavender}{cmyk}{0,0.5,0,0}

\definecolor{darkblue}{rgb}{0,0,0.5}

\definecolor{orange}{rgb}{1,0.5,0}

\else

\fi

\ifx\black\undefined
\newcommand{\black}{\color{black}}
\fi

\ifx\blue\undefined
\newcommand{\blue}{\color{blue}}
\fi

\ifx\red\undefined
\newcommand{\red}{\color{red}}
\fi

\ifx\green\undefined
\newcommand{\green}{\color{green}}
\fi

\ifx\white\undefined
\newcommand{\white}{\color{white}}
\fi

\ifx\magenta\undefined
\newcommand{\magenta}{\color{magenta}}
\fi

\usepackage{amssymb}
\usepackage{amsmath}
\allowdisplaybreaks[3] 
\usepackage{bbold}

\usepackage{flafter} 
\PassOptionsToPackage{pdftex}{graphicx} 

\usepackage[pdftex]{graphicx}

\usepackage{fancyhdr}
\usepackage{ifthen}
\usepackage{layout} 
\usepackage{xspace}% 
\usepackage{afterpage}% 
                     
\usepackage{url} 
\usepackage{textcomp} 
\usepackage{pdfpages} 
\usepackage{wasysym}

\usepackage[weather]{ifsym}

\def\ba#1\ea{\begin{align*}#1\end{align*}} 
\def\banum#1\eanum{\begin{align}#1\end{align}} 

\newcommand{\bi}{\begin{itemize}}
\newcommand{\ei}{\end{itemize}}
\newcommand{\be}{\begin{enumerate}}
\newcommand{\ee}{\end{enumerate}}
\newcommand{\bc}{\begin{center}}
\newcommand{\ec}{\end{center}}

\newcommand{\biq}{%bi-quetsch
\begin{list}{$\bullet$}{%
\setlength{\topsep}{0cm}
\setlength{\partopsep}{-\parskip}
\setlength{\leftmargin}{0.8cm}
 \setlength{\itemsep}{0pt}
    \setlength{\parskip}{0pt}
    \setlength{\parsep}{0pt}    
}}
\newcommand{\eiq}{\end{list}}

\makeatletter
\@ifclassloaded{beamer}{}{}
\makeatother

\newcommand{\RR}{\mathbb{R}}
 
\newcommand{\Nat}{\mathbb{N}}

\let\R\undefined 
\newcommand{\R}{\RR}
\newcommand{\N}{\Nat}

\newcommand{\Ccal}{\mathcal{C}}
\newcommand{\Dcal}{\mathcal{D}}
\newcommand{\Ecal}{\mathcal{E}}
\newcommand{\Fcal}{\mathcal{F}}
\newcommand{\Gcal}{\mathcal{G}}
\newcommand{\Hcal}{\mathcal{H}}

\newcommand{\Lcal}{\mathcal{L}}

\newcommand{\Pcal}{\mathcal{P}}

\newcommand{\Tcal}{\mathcal{T}}
\newcommand{\Ucal}{\mathcal{U}}

\newcommand{\Xcal}{\mathcal{X}}

\DeclareMathOperator*{\argmin}{argmin} % 

\DeclareMathOperator{\sign}{sign}
\newcommand{\eps}{\ensuremath{\varepsilon}}
\renewcommand{\epsilon}{\ensuremath{\varepsilon}}
\renewcommand{\phi}{\ensuremath{\varphi}}

\newcommand{\condon}{\; \big| \;}%
\let\Pr\undefined 
\DeclareMathOperator{\Pr}{P}

\usepackage{amssymb}
\usepackage{pifont}

\wasyfamily
\DeclareFontShape{U}{wasy}{b}{n}{ <-10> ssub * wasy/m/n
<10> <10.95> <12> <14.4> <17.28> <20.74> <24.88>wasyb10 }{}

\DeclareFontShape{U}{wasy}{m}{n}{ <5> <6> <7> <8> <9> gen * wasy
<10> <10.95> <12> <14.4> <17.28> <20.74> <24.88> <35> <40> <50> <60> wasy10  }{}

\newcounter{ulestheorem}

\newcommand{\ulestheorem}[3]{%
\refstepcounter{ulestheorem}%
\begin{ulesblock}{%
Theorem \arabic{ulestheorem}
\label{#1}%
\ifthenelse{\equal{#2}{}}{}{(#2)}}%title of the theorem, if exists
#3 % body of the theorem
\end{ulesblock}%
}

\newcommand{\ulesproposition}[3]{%
\refstepcounter{ulestheorem}%
\begin{ulesblock}{%
Proposition \arabic{ulestheorem}
\label{#1}%
\ifthenelse{\equal{#2}{}}{}{(#2)}}%title of the theorem, if exists
#3 % body of the theorem
\end{ulesblock}%
}
\newcommand{\ulescorollary}[3]{%
\refstepcounter{ulestheorem}%
\begin{ulesblock}{%
Corollary \arabic{ulestheorem}
\label{#1}%
\ifthenelse{\equal{#2}{}}{}{(#2)}}%title of the theorem, if exists
#3 % body of the theorem
\end{ulesblock}%
}

\usepackage{mathtools}
\usepackage{tcolorbox}
\usepackage{subcaption}

\usepackage{authblk}  % for author/institute formatting
\usepackage{footmisc} % for multiple footnotes

\newcommand{\ex}{\Ecal}
\newcommand{\Ex}{\mathfrak{E}}
\newcommand{\exgrad}{\Ecal_{\mathrm{grad}}}
\newcommand{\exshap}{\Ecal_{\mathrm{SHAP}}}
\newcommand{\exanchor}{\Ecal_{\mathrm{anchor}}}
\newcommand{\Exanchor}{\mathfrak{E}_{\mathrm{anchor}}}
\newcommand{\exwcf}{\Ecal_{\mathrm{wcf}}} % weak counterfactuals
\newcommand{\Exwcf}{\mathfrak{E}_{\mathrm{wcf}}}
\newcommand{\exscf}{\Ecal_{\mathrm{scf}}} % strong counterfactuals
\newcommand{\Exscf}{\mathfrak{E}_{\mathrm{scf}}}

\DeclareMathOperator{\E}{E}

\newcommand{\dd}{\mathrm{d}}

\DeclareMathOperator{\coverage}{coverage}
\DeclareMathOperator{\precision}{precision}

\newcommand{\Fe}{\Fcal_{\mathrm{explain}}}
\newcommand{\Fp}{\Fcal_{\mathrm{predict}}}
\definecolor{violet}{RGB}{150, 0, 150}

\usepackage{amsmath}
\usepackage{amssymb}
\usepackage{mathtools}
\usepackage{amsthm}

\theoremstyle{plain}
\newtheorem{theorem}{Theorem}[section]
\newtheorem{proposition}[theorem]{Proposition}

\newtheorem{lemma}[theorem]{Lemma}

\theoremstyle{definition}
\newtheorem{definition}[theorem]{Definition}

\theoremstyle{remark}

\newtheoremstyle{nonumberdot}%
  {}{}{\itshape}{}{\bfseries}{}{ }% 
  {\thmname{#1}\thmnumber{ #2}\thmnote{ (#3)}}

\theoremstyle{nonumberdot}
\newtheorem{propnodot}[theorem]{Proposition}

\usepackage[a4paper,top=2cm,bottom=2cm,left=3cm,right=3cm]{geometry}

\usepackage[hidelinks]{hyperref}

\usepackage{tcolorbox}
\tcbuselibrary{theorems, skins, breakable}

\usepackage{mdframed}

\begin{document}

\title{When Are Local Post-hoc Explanations Informative?}

\author[$*$,1]{Eric Günther}
\author[$*$,2]{Balázs Szabados}
\author[1]{Robi Bhattacharjee}
\author[$\dagger$,1]{Sebastian Bordt}
\author[$\dagger$,1]{Ulrike von Luxburg}

\affil[1]{University of Tübingen and Tübingen AI Center}
\affil[2]{HUN-REN Institute for Computer Science and Control (SZTAKI) and Department of Probability Theory and Statistics,
Institute of Mathematics, Eötvös Loránd University (ELTE), Budapest, Hungary\protect\\
\texttt{\{eric.guenther, sebastian.bordt, ulrike.luxburg\}@uni-tuebingen.de, balazs.szabados@sztaki.hu, robibhatt@gmail.com}}

\date{\today}

\maketitle

\begingroup
\renewcommand\thefootnote{}\footnotetext{*Equal contribution, ${\dagger}$Equal advising}
\endgroup

\begin{abstract}
    Local post-hoc explanation algorithms are widely used to gain insight into the behavior of complex machine learning models.
		However, theoretical guarantees about such algorithms only exist for simple decision functions, 
		and it is unclear whether and under which assumptions similar results might exist for complex models. 
		In this paper, we introduce a theoretical framework to formalize when an explanation provides information about a decision function. We call an explanation informative if it serves to restrict the %
		prediction behavior of the space of plausible decision functions. With this approach, we show that many popular explanation algorithms are not informative when applied to complex decision functions. 
        For example, gradient and counterfactual explanations are uninformative on the space of differentiable functions, while SHAP and anchor explanations are uninformative on the space of decision trees. Conversely, the same methods become informative on suitably restricted classes of simple functions, such as functions with low gradients and curvatures, or bounded-depth trees. 
\end{abstract}

\section{Introduction}
Modern machine learning systems significantly impact people’s lives. Automated systems determine whether someone receives a loan, is admitted to a graduate program, or is invited to a job interview \citep{uk_gov_review}. In such high-stakes applications of AI, many have argued that the affected individuals have the right to an explanation, 
and there is an ongoing debate regarding 
explainability and transparency requirements in legal frameworks such as the European General Data Protection Regulation 
\citep{WachterEtal17,hamon2022bridging,selbst2018meaningful}, the AI Act \citep{panigutti2023role,kaminski2025right}, and their combined implications \citep{hacker2020varieties,ebers2020regulating,nannini2024habemus,fresz2026explainable}.
Currently, the most prominent way of gaining an explanation is through local post-hoc explanation algorithms \citep{Lundberg17,sundararajan2017axiomatic,Ribeiro18}.
Ideally, one would like these explanations to be theoretically sound, grounded in the actual decision function, and with a mechanism to justify, verify, or contest them. 
Unfortunately, however, many results in the literature show the opposite: Explanation algorithms give rise to conflicting explanations, are sensitive to the choice of hyperparameters \citep{BorFinRaiLux22}, can be fooled and manipulated \citep{anders2020fairwashing,slack2020fooling,slack2021counterfactual}, and are hard to audit \citep{BhaLux2024}.
In this work, we argue that these negative findings do not originate from ``bad'' explanation algorithms. Instead, they are a fundamental consequence of the setup: Local post-hoc explanations are too limited a tool to reliably analyze arbitrarily complex decision functions. We express this intuition in the following informal conjecture: 
\begin{tcolorbox}[colback=gray!10, colframe=gray!50, title=]
{\bf Informal Conjecture}: A decision function can only admit useful local post-hoc explanations for all data points if the decision function is already simple.
\end{tcolorbox}
This paper formalizes and proves one version of this conjecture for the case of tabular data. 

\begin{figure*}[h]
     \centering
     \begin{subfigure}{\textwidth}
         \centering
         \includegraphics[width=\textwidth]{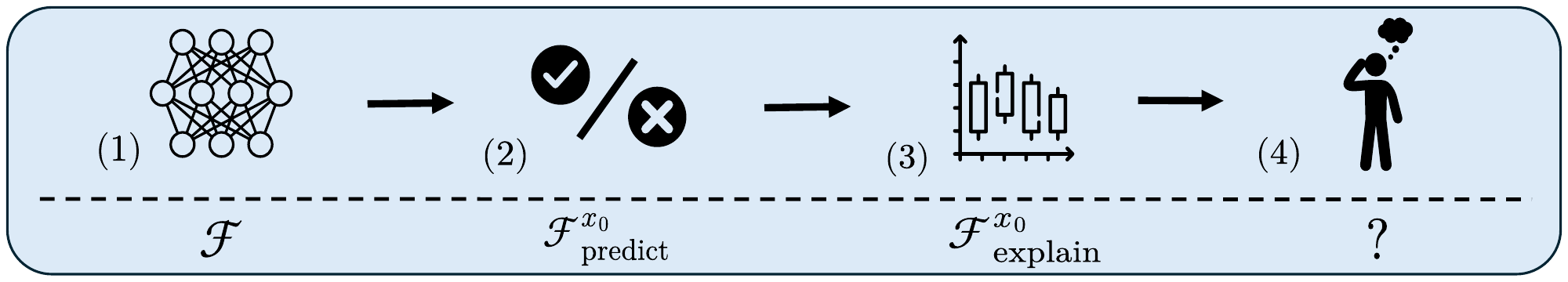}
     \end{subfigure}
        \caption{{\bf Conceptual overview of our approach:} (1) The user has background knowledge about the decision function, formalized as a function class $\Fcal$. (2) The user receives the prediction at point $x_0$. Now, the user knows that the function lies in the function class $\mathcal{F}_\text{predict}^{x_0}$. (3) 
        The user receives an explanation. Now, the user knows that the function must lie in the function class $\mathcal{F}_\text{explain}^{x_0}$. (4) We ask: How much has the user learned about the function $f$ after receiving the explanation? To answer this question, we study the relationship between $\mathcal{F}_\text{predict}^{x_0}$ and $\mathcal{F}_\text{explain}^{x_0}$. }
        \label{fig:illustration-intro}
\end{figure*}

\paragraph{Defining informative explanations. } To formalize the conjecture, the key challenge is to define what constitutes ``useful'' explanations. Intuitively, a minimal requirement for a useful explanation is that it tells us something new about the decision function that we cannot deduce by just looking at the prediction. In our new framework, we consider a decision function $f$ from a space $\Fcal$, a point $x_0$, and a local post-hoc explanation $\ex(f,x_0)$. We define $\Fp^{x_0}$ as the space of all functions in $\Fcal$ whose predictions on $x_0$ agree with $f$, 
and $\Fe^{x_0}$ as all functions that additionally get the same explanation. After receiving the prediction $f(x_0)$, we know that the decision function $f $ is in $\Fp^{x_0}$; after receiving the explanation, we even know that $f \in \Fe^{x_0}$. We call an explanation non-trivial if $\Fe^{x_0} \varsubsetneqq \Fp^{x_0}$: The explanation allows us to rule out at least one of the functions. We call an explanation {informative} if it allows to rule out a more substantial set of prediction behaviors: 
with positive probability over the choice of a random sample, the class 
$\Fe^{x_0}$ 
can realize fewer labelings on the sample than 
 $\Fp^{x_0}$.  
With this definition, we capture the intuition that upon receiving the explanation, a user should be able to 
derive some possibly limited, but new insight about the model's behavior  
\citep{miller2019explanation,jacovi2021formalizing,mohseni2021multidisciplinary,jalali2023predictability,
liao2021human,BorRaiLux25}.
See Figure \ref{fig:illustration-intro} for an illustration.

\paragraph{Main results.} 
With our novel definition of informative explanations, we prove the following results.
\be
\item[(1)] If we make no or weak assumptions on the decision function, then common post-hoc explanations are not informative. For example, gradient explanations and counterfactual explanations on the space of differentiable functions are not informative (Thm. \ref{thm-gradients} and \ref{thm-counterfactuals}), and SHAP explanations and anchors on the space of decision trees are not informative (Thm. \ref{thm-shap-trees} and \ref{thm-anchors}).

\item[(2)] If we introduce assumptions about the robustness or simplicity of the decision function, post-hoc explanations become informative. For example, gradient explanations are informative on spaces with low gradients and curvatures (Thm. \ref{thm-gradients});  counterfactual explanations are informative on Lipschitz functions (Thm. \ref{thm-counterfactuals}); SHAP explanations and anchors are informative on decision trees with small depth (Thm. \ref{thm-shap-trees} and \ref{thm-anchors}). 

\ee 

Our mathematical framework reveals that the distinction between uninformative explanations on complex function spaces versus informative ones on simpler spaces is quite general.  First, it arises across many popular local post-hoc explanation methods (gradients, SHAP, anchors, counterfactuals). Second, our negative results hold not only for specifically contrived decision functions, but for several common function classes (trees, continuous functions, differentiable functions,...) whenever there is no further bound on their complexity.
Third, as we show in Section \ref{sec-multiple-explanations}, our results are not specific to single-instance explanations but also extend to a setting of a finite number of explanations.

\paragraph{Consequences.} Result (1) establishes that if the underlying model class allows for arbitrarily complex decision functions, then no matter which explanation we receive for a particular data point, we cannot draw any sound conclusion from the explanation that concerns the prediction behavior of our model.
While this may not be a big problem in scenarios where explanations are used for 
exploratory data analysis or code debugging, it poses a challenge in safety-critical or legal settings where
recipients expect explanations that are reliable
and not misleading. 
Result (2) suggests that explanations can be improved by augmenting them with further information, for example, about the robustness or stability of the decision function or the explanation. Then explanations may become informative, and users can safely use them to draw conclusions.

\section{Related work}
The field of explainable machine learning has been rapidly growing over the last decade; for an overview, see \citet{molnar2025}. Existing research can be broadly distinguished into research on interpretable models \citep{caruana2015intelligible,letham2015interpretable}, post-hoc explanation algorithms \citep{Lundberg17,Ribeiro16}, and mechanistic interpretability \citep{nanda2023progress}. In this work, we focus on post-hoc explanation algorithms.
Post-hoc explanations are intuitively appealing because they are applicable to any model, meaning that we do not need to sacrifice accuracy \citep{Lundberg17}. Consequently, post-hoc explanations have been heralded for a variety of important and often high-stakes use-cases, including 
model debugging \citep{adebayo2020debugging}, 
detecting discrimination \citep{agarwal2018automated}, 
giving recourse recommendations \citep{karimi2022survey}, 
applications of the AI Act \citep{panigutti2023role}, 
and many others. 
At the same time, there is an ongoing debate about the usefulness of post-hoc explanations. Most prominently, \citet{rudin2019stop} called for a stop of using these methods because they can be ``misleading and often wrong'' \citep[Section~5]{rudin2019stop}, a possibility that often arises due to the existence of model multiplicity \citep{Marx_2020_PredictiveMultiplicity,Breiman_2001_StatisticalModeling,Black_2022_ModelMultiplicity,rudin2024amazing}, but also due to instability of the explainer itself \citep{hwang2026explanation}. \citet{BorFinRaiLux22} find that post-hoc explanations do not serve their purpose in many legal contexts, \citet{Freiesleben_Meding_König_2026} show that they are insufficient to contest an algorithmic decision.

Indeed, there exist many negative results regarding local post-hoc explanations in the literature. For example, it is known that local explanation algorithms are sensitive to parameter choices, and different explanation algorithms tend to give widely different results 
\citep{BorFinRaiLux22,krishna2024disagreement}. This can be specifically exploited 
to co-develop machine learning models that always come with ``harmless'' explanations \citep{sharma2024x}. 
Moreover, many local explanation algorithms can be manipulated and fooled by malicious explanation providers 
\citep{slack2020fooling,anders2020fairwashing,slack2021counterfactual,tekkesinoglu2026explanations}. But at the same time, auditing local explanations in hindsight is hard \citep{BhaLux2024}.

In terms of theory, several works offer supporting evidence for our conjecture, though none directly formalize it. \citet{bressan2024theory} characterizes regimes in which complex decision functions can be approximated by simpler models, while \citet{bilodeau2024impossibility} shows that feature attributions often fail to test properties of the decision function in general settings. \citet{HillMasoomi24} quantifies how local explanation uncertainty grows with the geometric complexity of the decision boundary, suggesting that uniformly reliable attributions are more likely when the boundary is simple. Other works connect degraded attribution quality to complexity more indirectly: \citet{Kumar20} provide examples where Shapley-value attributions misrepresent feature contributions in the presence of intricate dependencies, \citet{Fokkema23} prove that attribution methods offering actionable recourse are increasingly vulnerable to manipulation as boundaries become more complex, and \citet{WangTheory22} show that adversarially robust models with smoother boundaries tend to yield more interpretable, counterfactual-aligned attributions.

The large number of different post-hoc explanation algorithms, as well as the mentioned failure cases, have led to the attempt to develop desirable criteria for ``good'' and ``trustworthy'' explanations (see \citealp{Nauta_2023} for an extensive overview). Particularly relevant for our work are notions of ``faithfulness'', which try to assert whether an explanation is grounded in the function and the data or not \citep{bhatt2021evaluating, Dasgupta22, velmurugan2021evaluatingfidelity}. 

While it is rarely explicitly discussed, the idea that explanations should provide additional information about the function (beyond the prediction) permeates the literature on post-hoc explanation algorithms (for an explicit mention, see e.g. Remark 3.3 in \citealp{azzolin2025beyond}, or \citealp{adebayo2018sanity}), see also 
\citep{miller2019explanation,jacovi2021formalizing,mohseni2021multidisciplinary,jalali2023predictability,
liao2021human,BorRaiLux25}. 
In many papers, this is expressed by describing the explanations as  ``interpretable approximations'' to the decision function \citep{Lundberg17}. Indeed, \citet{han2022explanation} has shown that many local explanation algorithms can be viewed from the perspective of local function approximation. %

Beyond the technical properties of explanation algorithms, an important line of work discusses whether and how explanations can be fruitfully used by their human recipients \citep{miller2019explanation,liao2021human,hullman2025explanationsmeansend}. 
It has often been pointed out that it is unclear what purpose the explanations are supposed to serve \citep{lipton2018mythos, molnar2020general,krishnan2020against,BorRaiLux25,hullman2025explanationsmeansend}.

An interesting connection exists between our results on gradient explanations and related work in computer vision. In image classification, it has been observed that the input gradients of unconstrained models are often noisy and uninterpretable \citep{smilkov2017smoothgrad}, whereas those of robust models are frequently human-interpretable \citep{tsipras2018robustness,kaur2019perceptually,srinivas2023models}. What is more, \citet{serrurier2023explainable} trains 1-Lipschitz neural networks and argues that they exhibit desirable explainability properties. These empirical observations connect neatly to our Theorem \ref{thm-gradients}. It will be a matter of future work to closely investigate the implications of our results for image classification and its explanations. 

\section{Notation and setup}
\label{sec-setup}

\subsection{General setup} 
In this paper, we deal with classification or regression on tabular data. We consider an input space $\Xcal \subseteq \R^d$ with $d$ dimensions and probability measure $\Pr$, usually coming from a density. We denote the topological interior of $\Xcal$ by $\operatorname{int}(\Xcal)$. For a natural number $n\in\N$, we write $[n]:=\{1,...,n\}$. For a subset $J \subseteq [d]$, we denote its complement by $\bar{J} $. 
For points $x\in\R^d$, we denote their coordinates with superscripts, that is 
$x = (x^{(1)}, ... , x^{(d)})$.   For a subset $J=\{j_1,...,j_{|J|}\}\subseteq[d]$ of features and a point $x \in \R^d$, we write $x^{(J)}:=(x^{(j_1)},...,x^{(j_{|J|})})$, indicating that only the features $J$ of $x$ are considered. We consider functions $f$ on $\Xcal$ with real-valued output. The functions come from some function class $\Fcal$. For a subset $S\subseteq\Xcal$, we denote the restriction of a function $f$ on $S$ by $f|_S:S\to\R, x\mapsto f(x)$ and the restriction of the function class $\Fcal$ to $S$ by $\Fcal|_S:=\{f|_S\mid f\in\Fcal\}$. In case of regression, we directly output the function value $f(x)$. In case of binary classification, we interpret $f$ as a scoring function, usually with values in $[-1,1]$, from which we derive the final output label by $c_f(x) := \sign(f(x))$. We write $\Fcal^{[-1,1]}:=\{f\in\Fcal\mid f(\Xcal)\subseteq [-1,1]\}$ for the functions in $\Fcal$ that only take values in $[-1,1]$. For notational simplicity, we restrict ourselves to binary classification, but our results carry over to a multi-class setting as well. 

\subsection{Explanation algorithms}

This paper deals with local post-hoc explanation algorithms. Such an algorithm gets a function $f \in \Fcal$ and a point $x_0 \in \Xcal$ as input and produces an explanation $\ex(f,x_0)$ whose concrete form depends on the explanation algorithm. Other factors, like the probability distribution, might also influence the explanation, but are omitted in the notation for better readability. Some algorithms (gradients, SHAP) produce a unique explanation $\ex(f,x_0)$, whereas others (anchors, counterfactuals) might not. In such cases, we write $\Ex(f,x_0)$ for the set of all explanations and $\ex(f,x_0)\in\Ex(f,x_0)$ for a specific one. We consider the following explanations:

\paragraph{Gradient explanations.} 
We denote the explanation that returns the gradient of the function $f$ at point $x_0$ by $\exgrad(f,x_0) := \nabla f(x_0)$. In practice, many algorithms compute variants of gradients, for example, LIME~\citep{Ribeiro16,GarLux20}, SmoothGrad \citep{smilkov2017smoothgrad}, or integrated gradients \citep{sundararajan2017axiomatic}. Many of our impossibility results regarding gradient explanations directly carry over to any explanation that can be derived from the raw gradient. 

\paragraph{SHAP explanations.}
Inspired by game theory, the popular SHAP explanation algorithm~\citep{Lundberg17} 
computes the importance of a feature $j$ for the prediction  $f(x)$ based on a {value function} $v:\Pcal([d]) \times \Xcal \to \R$, where $\Pcal$ denotes the power set. It measures the ``average value'' of the function $f$ if a subset $J\subseteq[d]$ of the features of $x$ is fixed. In this paper, we focus on the widely used interventional value function
$v(J,x):=\E_{X^{(\bar{J})}}(f(x^{(J)},X^{(\bar{J})})),$
which averages out the remaining features using the marginal expectation. 
The final SHAP importance score $\Phi_j$ of feature $j$ is then given by 
$$\Phi_j(x) = \sum_{J\subseteq [d]\setminus j} \frac{|J|!\cdot(d-|J|-1)!}{d!}\cdot (v(J\cup j, x)-v(J,x)).$$ 
We denote the corresponding explanation by $\exshap(f,x_0) := (\Phi_1(x_0), ..., \Phi_d(x_0))$. We ignore any computational or statistical difficulties in evaluating these values and assume that the true SHAP values are reported.

\paragraph{Anchor explanations.}

Anchors~\citep{Ribeiro18} explain the classification of a data point in terms of a simple rule that relates to the features of the point. 
For simplicity, we consider binary classification: a scoring function $f$ with values between $-1$ and $1$, where the final classification output is given by $c_f(x) = \sign(f(x))$. 
A rule $R$ consists of statements like
$
 \; x^{(1)} \geq a_1 \;\; \&\;\; b_1 \leq x^{(2)} \leq b_2.
$
Intuitively, such a rule is supposed to outline a box in the data space (possibly unbounded in some directions) which contains $x_0$, such that most of the points in this box have the same label as $x_0$. 
For a rule $R$, denote by $\Xcal_R$ the set of all points in $\R^d$ that satisfy the rule. The coverage of a rule $R$ is defined as the probability that a randomly chosen point would satisfy the rule, so $\coverage(R) := \Pr(\Xcal_R)$. The precision of a rule $R$ is defined as the probability that points that satisfy the rule are assigned the same label as $x_0$. We thus define
\ba
\precision(R,f, x_0) := 
\frac{\Pr\left(\left\{x \in \Xcal_R\condon c_f\left(x\right) = c_f(x_0)  \right\}\right)}{\Pr(\Xcal_R)}.
\ea
An anchor explanation is defined as a tuple 
$\exanchor(f,x_0)=(R,p),$
where $R$ is the rule and $p$ its precision. We call an anchor \emph{perfect} if its precision is $1$. We denote the set of all possible anchors with precision $p$ around $x_0$ for function $f$ by
$$
\Exanchor(f,x_0,p):=\{(R,p)\condon x_0\text{ satisfies } R,
\precision(R,f,x_0)=p\}.
$$

\paragraph{Counterfactual explanations.}
Counterfactual explanations~\citep{WachterEtal17} are typically used in a classification scenario and are defined as the closest point to $x_0$ with a different label. For simplicity, we consider binary classification: a scoring function $f$ with values between $-1$ and $1$, where the final classification output is given by $c_f(x) = \sign(f(x))$. We consider the optimization problem %
\begin{equation}\label{eq-counterfactual-strong}
\begin{aligned}
	\argmin_{v\in\mathbb{R}^d} \quad  \|v\| \quad
	\text{s.t.} \quad  c_f(x_0)\neq c_f(x_0+v).
\end{aligned}
\end{equation}
If $v^*$ is a minimizer of (\ref{eq-counterfactual-strong}), we call the corresponding explanation $\exscf(f,x_0)=x_0+v^*$ a {\bf strong counterfactual} explanation (scf standing for ``strong counterfactual''). Here we gloss over the fact that (\ref{eq-counterfactual-strong}) might not be well-defined, see Appendix \ref{appendix-strong-counterfactuals}. 
Strong counterfactuals are not necessarily unique because (\ref{eq-counterfactual-strong}) might not have a unique solution. We denote the set of all strong counterfactuals by $\Exscf(f, x_0)$. 
In practice, counterfactuals are not guaranteed to solve (\ref{eq-counterfactual-strong}) as they are the result of a heuristic search algorithm. 
A \textbf{weak counterfactual} $x_C=\exwcf(f,x_0)$ is a point such that $c_f(x_C)\ne c_f(x_0)$. 
We denote the set of all weak counterfactuals by
$$\Exwcf(f,x_0):=\{x\in\Xcal\condon c_f(x)\neq c_f(x_0)\}.$$
There are other possibilities of defining counterfactuals, which do not satisfy (\ref{eq-counterfactual-strong}). However, without the minimal-distance-guarantee, they behave like weak counterfactuals, see Appendix \ref{subsec-alternative-cf-def}.

\section{Informative explanations}
We now introduce the key definitions for this paper. They capture whether an explanation $\ex(f,x_0)$ contains any information about the function $f$ that we cannot derive from the prediction $f(x_0)$ alone. We assume that $f$ comes from a class $\Fcal$ that encodes our background knowledge on the decision function, for example, the class of all random forests. 
We start by considering all functions in $\Fcal$ that produce the prediction $f(x_0)$. We call an explanation $\ex(f,x_0)$ for $f\in\Fcal$ trivial if it does not further reduce this space. A non-trivial explanation rules out at least one function. However, this might still not give us any tangible insight. For an explanation to be informative, it must restrict the possible prediction behavior of the class.

\subsection{Definition of non-trivial and informative explanations}
\label{Definition Informativeness}
Consider an input space $\Xcal \subseteq \R^d$ and a space $\Fcal$ of real-valued functions defined on $\Xcal$. Consider an explanation algorithm that for all $f \in \Fcal$ and for all $x \in \Xcal$ produces a unique explanation $\ex(f,x)$. Fix $f \in \Fcal$ and $x_0 \in \Xcal$ and consider the explanation $\ex(f, x_0)$ of function $f$ at point $x_0$. 
Define the following two spaces of functions: 
$\Fp^{x_0}$, the space of all functions that agree with $f$ on the prediction $f(x_0)$; and 
$\Fe^{x_0}$, the functions that agree with $f$ both in terms of prediction $f(x_0)$ and the corresponding explanation: 
\ba
& \Fp^{x_0} := \{ g \in \Fcal \,|\, g(x_0) = f(x_0)\}\\
& \Fe^{x_0} := \{ g \in \Fcal \,|\, g(x_0) = f(x_0),\ex(g,x_0) = \ex (f, x_0)\}.
\ea
In the case of non-unique explanations, we define $\Fe^{x_0}$ to be the set of all functions that agree with $f$ on the prediction and for which the given explanation $\ex(f,x_0)$ is valid:
$$\Fe^{x_0}:=\{ g \in \Fcal\,|\, g(x_0) = f(x_0),\ex(f,x_0) \in \Ex (g, x_0)\}.$$
We say that the explanation $\ex(f,x_0)$ is {\bf non-trivial} with respect to $\Fcal$ if
$\Fe^{x_0} \varsubsetneqq \Fp^{x_0}$. \\
We call an explanation $\ex(f,x_0)$ \textbf{informative} with respect to $\Fcal$ and a fixed value $n\in\Nat$ if with positive probability over the choice of a random sample $S=(x_1,...,x_n)$, there exists some $g\in\Fp^{x_0}$ such that every $\tilde{g}\in\Fe^{x_0}$ differs from $g$ on at least one of the sample points: 
$$\mathrm{Pr}_{S\sim \Pr^n}\left(\Fe^{x_0}|_S \varsubsetneqq\Fp^{x_0}|_S\right)>0.$$
Conversely, an explanation is uninformative with respect to $\Fcal$ and $n\in\Nat$ if for almost surely every $x_1,...,x_n$ and every $g\in\Fp^{x_0}$, there is some $\tilde{g}\in\Fe^{x_0}$ yielding $g(x_i)=\tilde{g}(x_i)$ for $i=1,...,n$, or in short,
$$\mathrm{Pr}_{S\sim \Pr^n}\left(\Fe^{x_0}|_S =\Fp^{x_0}|_S\right)=1.$$

We call an explanation algorithm $\ex$ (resp. $\Ex$) \textbf{informative} for $\Fcal$ if for all $f\in\Fcal$ and almost surely all $x_0\in\Xcal$, there exists an $N\in\N$ such that $\ex(f,x_0)$ (resp. all $\ex(f,x_0)\in\Ex(f,x_0)$) is informative with respect to all $n\ge N$. The number $N$ can be seen as the minimal sample size needed to distinguish between the classes.\\
We call an explanation algorithm $\ex$ (resp. $\Ex$) {\bf fully uninformative} for $\Fcal$ if for all $f\in\Fcal$ and almost surely all $x_0\in\Xcal$, $\ex(f,x_0)$ (resp. all $\ex(f,x_0)\in\Ex(f,x_0)$) is not informative for any $n\in\N$.

\subsection{First examples}
\label{subsection-first-examples}
To gain intuition for these definitions, we now discuss some simple examples. In Appendix \ref{sec-definition-choices} we include a longer discussion about the design decisions and rationales of our definitions.
\paragraph{Examples of trivial explanations.}
Assume a model developer provides an explanation that is independent of the function $f$, for example, the largest coordinate of $x_0$. Such an explanation is always trivial and not informative because all $g\in\Fcal$ have the same explanation at $x_0$ and hence $\Fe^{x_0} = \Fp^{x_0}$. 

As another example, consider the tautological explanation which simply returns $\ex(f,x_0)=f(x_0)$. Although this is a perfectly faithful explanation for $f$, it is obviously useless. It is clear that $\Fp^{x_0}=\Fe^{x_0}$, so the explanation is trivial with respect to our definition.

\paragraph{Gradient explanations on linear functions are informative. }
Consider $\Xcal = [0,1]^d$ with the uniform distribution,  the class $\Fcal$ of affine linear functions on $\Xcal$, and gradient explanations $\ex(f,x_0) = \nabla f(x_0)$. 
For $x_0\in\Xcal$, the space $\Fp^{x_0}$ contains all functions $g$ with $g(x_0)=f(x_0)$. If we additionally know the gradient of such a function $g$, it is completely determined as it is affine linear. Hence, $\Fe^{x_0}=\{f\}$ only contains a single function and can clearly not emulate the same predictions on a sample as $\Fp^{x_0}$. Hence, the explanation is informative. See Proposition \ref{prop-gradient-linear-and-noisy} below for a formal statement.

\paragraph{Gradient explanations for piecewise constant functions are non-trivial, yet not informative.}
Let $\Xcal = [0,1]^d$ and assume that the probability distribution $\Pr$ has a positive density. Consider the space $\Fcal$ of all axis-parallel decision trees on $[0,1]^d$ with depth 3 and let
$$\ex(f,x_0):=\begin{cases}
    \nabla f(x_0),\quad &\text{if $f$ is differentiable in $x_0$}\\
    \mathrm{NaN}, &\text{else}.
\end{cases}$$ 
One would not expect gradient explanations to contain any information on trees, as they are piecewise constant, meaning that with probability 1 over the choice of $x_0$, it holds $\ex(f,x_0)=0$. However, if $\ex(f,x_0)=0$, we know that the gradient is well-defined in $x_0$, so there is no cut passing through $x_0$. In particular $\Fe^{x_0}\varsubsetneqq\Fp^{x_0}$, because trees with cuts through $x_0$ are excluded. Nevertheless, one might consider this information worthless, as the function space is still almost the same. Indeed, this intuition is perfectly captured by our definition of informativeness: For a finite sample, a tree can still produce the same predictions without having cuts passing through $x_0$, see Appendix \ref{appendix-gradient-for-trees} for a formal proof. Hence, the explanation is non-trivial but not informative. The example illustrates why the strict inclusion $\Fe^{x_0}\varsubsetneqq\Fp^{x_0}$ is not strong enough to capture whether the explanation actually tells us something useful. In particular, distinguishing between non-trivial and informative explanations is crucial.

\paragraph{Whether an explanation is informative depends on our background knowledge $\Fcal$.}
The class $\Fcal$ does matter for whether an explanation is informative as it encodes the prior knowledge, and hence, also whether the explanation tells us something new or not. Let us consider gradient explanations: Gradient explanations are informative on the space of all polynomials with bounded degree and coefficients, but they become uninformative on the much larger space of all polynomials (Thm. \ref{thm-gradients}). This is because the latter space is too complex to be reduced by a single known gradient. The other way round, the explanation might also become uninformative if the space $\Fcal$ becomes much smaller: Gradient explanations are also uninformative on the much smaller space of affine linear functions with gradient $(1,1,...,1)$, as the space is so small that all functions receive the same explanation. Also, from an intuitive point of view, one would expect a gradient to be useless if it is already known that we trained functions of this form.

\paragraph{Already small local restrictions can turn explanations informative.}
Note that even if our explanation just restricts the function behavior in a tiny neighborhood around $x_0$, this leads to an informative explanation: As long as this region has positive probability mass, the probability of drawing a sample on which the function classes do not match is positive.

\section{Gradient explanations}\label{sec-gradient-explanations}
\label{section-gradients}
Now, let us move on to study gradient explanations in detail. In this section, we consider functions $f:\Xcal \to \R$, where $\Xcal\subseteq\R^d$, which can either be used in a regression setup or serve as the scoring function $f$ in classification. Typically,  we assume that the probability distribution on $\Xcal$ has a positive density.
\subsection{Examples: linear vs. noisy linear functions}
To begin, we will consider linear functions $\Lcal^M$ and contrast them with noisy linear functions 
$\mathcal{L}^M_\varepsilon$ (linear functions on which we add tiny ``wiggles'' of size at most $\varepsilon$, which do not affect their behavior on the macro scale):
\begin{align*}
\Lcal^M&:=\big\{f:\Xcal\to\mathbb{R}\condon f(x)=w^Tx+b, \ w \in \R^d, \ b \in \R, \ \|w\|,|b|<M\big\}\\
\mathcal{L}^M_\varepsilon& :=
\left\{
f:\Xcal\to\mathbb{R}\condon 
f \text{ differentiable},\, \exists g\in\Lcal^M 
\forall x\in\Xcal: |f(x)-g(x)|< \varepsilon 
\right\}.
\end{align*}

\begin{proposition}[\textbf{Gradient explanations are informative on linear functions but not on noisy linear functions}]
\label{prop-gradient-linear-and-noisy}
Consider a compact data space $\Xcal\subseteq\R^d$ with a probability distribution $\Pr$ that has a positive density. \begin{enumerate}
    \item Gradient explanations $\exgrad$ are informative for linear functions $\Lcal^M$.
    \item Gradient explanations $\exgrad$ are fully uninformative for noisy linear functions $\Lcal^M_\eps$.
\end{enumerate}
\end{proposition}
\textbf{Proof sketch } (full proof in Appendix \ref{proof-prop-gradient-linear}). In the case of linear functions, the gradient and $f(x_0)$, determine the function completely, so $\Fe^{x_0}=\{f\}$, restricting the prediction behavior more strongly than that of $\Fp^{x_0}$. For noisy linear functions, given any $g\in\Fp^{x_0}$ and a sample $x_1,...,x_n$, we can change the gradient of $g$ at $x_0$ by changing $g$ in a small neighborhood, leaving its predictions on $x_1,...,x_n$ unchanged. This delivers a $\tilde{g}\in\Fe^{x_0}$ with the same prediction behavior on that sample. \hfill \qedsymbol

This is exactly the result one would expect. Gradient explanations are informative for linear functions because they already determine the function. On the other hand, they are not informative on noisy linear functions: The noise makes the gradient arbitrary, so it no longer contains relevant information about the predicting behavior of our function. 
Similar statements can also be proved for piecewise linear functions, see Appendix \ref{section-gradient-piecewise-linear}. Also, the results carry over if only the largest component of the gradient is known, see Appendix \ref{appendix-gradient-largest-component}. 
The proposition already suggests that gradient explanations are only informative if the function space does not contain strong variation.
\subsection{Gradients are only informative on small function classes}
\label{subsection-gradient-on-complex-spaces}
In the following, we consider gradient explanations on less trivial function classes, proving one of our main results. For polynomials, we use the multi-index notation: For $\nu\in\N_0^d$, we write $x^\nu:=(x^{(1)})^{\nu_1}\cdots(x^{(d)})^{\nu_d}$ and $|\nu|:=\nu_1+...+\nu_d$.
\begin{theorem}[\textbf{Gradients are only informative if they cannot change arbitrarily fast}]\label{thm-gradients}
    Consider a compact data space $\Xcal\subseteq\R^d$ with a probability distribution $\Pr$ that has a positive density. 
    \begin{enumerate}
        \item  Gradient explanations $\exgrad$ are fully uninformative for the class $\Dcal$ of differentiable functions, all continuously differentiable functions with bounded gradient
        $
        \Dcal_{\alpha}= \{f\in C^1(\Xcal, \R)\,|\,  \forall x\in\Xcal: \|\nabla f(x)\|<\alpha \}$, 
        and all polynomials
        $\mathcal{P}=\{f:\Xcal\to\R\,|\, f(x)=\sum_{|\nu|\leq k}a_\nu x^\nu, k\in\mathbb{N}\}$.
        \item  Gradient explanations $\exgrad$ are informative for differentiable functions with bounded and Lipschitz-continuous gradients 
        $\Dcal_{\alpha,\beta}:=\{f\in \Dcal_\alpha\condon \forall x,y\in\Xcal:
        \|\nabla f(x)-\nabla f(y)\|\le \beta\|x-y\|\}$
        and all polynomials with bounded degree and bounded coefficients 
        $\mathcal{P}_{D,M}=\{f:\Xcal\to\R\condon f(x)=\sum_{|\nu|\leq D}a_\nu x^\nu,\, \forall \nu\neq(0,0,\dots,0):
        |a_\nu|\le M\}$.
    \end{enumerate}
\end{theorem}

\textbf{Proof sketch} (full proof in Appendix \ref{app:proof-thm-gradients}). Given $g\in\Fp^{x_0}$ and a sample $x_1,...,x_n$, we can modify $g$ in a small neighborhood around $x_0$ to get the desired gradient while leaving the predictions on $x_1,...,x_n$ the same. This gives a function $\tilde{g}\in\Fe^{x_0}$ with the same prediction behavior on the sample. This idea works for all classes in Part 1. \\
In Part 2, the gradients are Lipschitz-continuous and can no longer change arbitrarily fast. Hence, a change in the gradient of $g$ at $x_0$ extends to a neighborhood, which restricts the prediction behavior there. \hfill\qedsymbol

All function classes in the first part of the theorem have in common that the gradient can vary arbitrarily fast, so a single gradient does not contain any information about the function away from our point. In the proofs, we exploit that despite a gradient restriction in $x_0$, we can still construct functions to emulate an arbitrary behavior if we move away from $x_0$ just a little bit. 
If we want gradients to be informative, we hence need to consider a function space where the gradient can only change moderately fast. This is true for Lipschitz gradients and polynomials with bounded degree and coefficients in the second part of the theorem. Alternatively, instead of restricting the function class, we may require the gradient to be locally stable around our point $x_0$. This also results in informativeness, see Appendix \ref{appendix-gradients-locally-stable}.

\section{SHAP explanations}
We now turn our attention to SHAP explanations. For SHAP explanations, it is much less obvious whether explanations are informative or not. Unlike gradients, which only depend on an arbitrarily small neighborhood of $x_0$, SHAP explanations depend on a larger part of the space. For a better intuition on how SHAP explanations restrict the function space, see Appendix \ref{appendix-shap-intuition}.

We consider the scenario of regression with real-valued functions $f$ or classification where the classification $c_f$ is based on thresholding the function $f$. Note that even in the classification case, we compute the SHAP values from $f$ directly, as it is the standard approach for SHAP. 

\subsection{The illustrative case of grid functions: full support vs. manifold setting}
To gain intuition, let us start again with simple examples: piecewise constant functions on a grid. Consider the unit cube $\Xcal = [0,1]^d$ with an equidistant grid with $k-1$ subdivisions along each axis, leading to $k^d$ grid cells. Let $\Hcal_k$ be the set of functions that are piecewise constant on these $k^d$ grid cells and have values in $[-1,1]$. We denote the space of mean-centered grid functions by $\bar{\Hcal}_k = \{f\in\Hcal_k \condon \E(f(X)) = 0\}$.
\begin{propnodot}\textup{(\textbf{SHAP explanations are informative on grid functions, except on manifolds})\textbf{.}}\label{prop-shap-grid}
\begin{enumerate}
    \item Let $\Xcal=[0,1]^d$ with a probability distribution $\Pr$ that has a positive density. Interventional SHAP explanations $\exshap$ are informative for grid functions $\Hcal_k$ with values in $[-1,1]$. 
    \item Let $\Xcal=[0,1]^2$ with data distributed uniformly on the diagonal $\{x\in\Xcal\condon x^{(1)}=x^{(2)}\}$. Then, interventional SHAP explanations $\exshap$ are fully uninformative for mean-centered grid-functions $\bar{\Hcal}_k$ with values in $[-1,1]$. 
\end{enumerate}
\end{propnodot}

\textbf{Proof Sketch} (full proof in Appendix \ref{app: proof-shap-grid}). There are $k^d$ grid cells. As soon as the number of sample points $S=(x_1,...,x_n)$ is big enough, we know that with positive probability there is one point in each grid cell. In this case, some $g\in\Fp^{x_0}$ is already fully determined by $g(x_1),...,g(x_n)$, meaning that a function $\tilde{g}\in\Fe^{x_0}$ matching its predictions on the sample must already equal $g$. So for $g\notin\Fe^{x_0}$, we may conclude $\Fe^{x_0}|_S\varsubsetneqq \Fp^{x_0}|_S$. However, if all the probability mass is on the diagonal, as in the second part, $\tilde{g}$ and $g$ only need to coincide there, but not necessarily off-manifold. As interventional SHAP values also depend on the function values off-manifold, we can construct a function $\tilde{g}$ with the desired SHAP values (so $\tilde{g}\in\Fe^{x_0}$) that still satisfies $\tilde{g}(x_i)=g(x_i)$.

The proposition outlines important mechanisms for SHAP explanations.
SHAP depends on the values of $f$ on other grid cells, so knowing $\exshap(f,x_0)$ restricts the values of $f$ in other cells, making the explanation informative.
In the second case, we can change the functions outside of the data manifold to receive the desired SHAP values while leaving the functions the same on-manifold. This means we can keep the same predictions while getting the desired explanation, so the explanation is not informative.

\subsection{SHAP explanations are only informative on small function classes}

We now turn our attention towards more realistic function classes. We denote by $\Tcal_\infty$ the set of all decision trees on $\R^d$ with axis-parallel splits, arbitrary depth, and values in $[-1,1]$. For $K \in \Nat$, denote by $\Tcal_K$ the subset of trees with values in $[-1,1]$ that have depth at most $K$. As a second example, we turn our attention to the class of generalized additive models (GAMs), which are functions of the form
$g(x) = \sum_{j=1}^d g_j (x^{(j)})$.
The functions $g_j$ are called the component functions of the GAM. It is well known \citep{BorLux_shap_2023} that interventional SHAP recovers the contributions of the component functions: $\Phi_j(x) = g_j(x^{(j)}) - \E(g_j(X^{(j)}))$. 
So in some sense, GAMs are the prime example of a function class where SHAP explanations should be considered successful; see also \citet{enouen2025instashap}. 
By $\Gcal_\infty$ we denote the class of GAMs with values in $[-1,1]$ whose component functions $g_j$ are decision trees with unbounded depth. 
\begin{theorem}[\textbf{SHAP is only informative on trees if their depth is bounded}] \label{thm-shap-trees} 
Consider $\Xcal=\R^d$  with a probability distribution $\Pr$ that has a positive density.
\begin{enumerate}
    \item Interventional SHAP  $\exshap$ is fully uninformative for the class $\Tcal_\infty$ of trees of arbitrary depth.
    \item Interventional SHAP  $\exshap$ is fully uninformative for the class $\Gcal_\infty$ of GAMs with trees of arbitrary depth as components.
    \item Interventional SHAP  $\exshap$ is informative for the class $\Tcal_K$ of trees of depth $\leq K$.
\end{enumerate}
\end{theorem}

\textbf{Proof sketch} (full proof in Appendix \ref{app:proof-shap}). The SHAP values in $x_0$ only depend on the values $f(x)$ for points $x$ that share at least one coordinate with $x_0$ and on $\E(f(X))$. 
Trees of unbounded depth can fit an arbitrarily big sample of points with arbitrarily small leaves, such that they do not contain any points that share coordinates with $x_0$ and still allow for any expected value. Given some $g\in\Fp^{x_0}$ and points $x_1,...,x_n$, we can hence take $f$, modify it such that it coincides with $g$ on the sample but still has the same SHAP values as before, leading to some $\tilde{g}\in\Fe^{x_0}$ with $g(x_i)=\tilde{g}(x_i)$. This implies uninformativeness.\\
A similar argument works for GAMs if the components are not restricted, leading to part 2. \\
In the case of bounded depth, a tree can only take $2^K$ many values. Once the number of sample points and $x_0$ exceeds $2^K$, a tree fitting them all has used up all its leaves. This partially determines the SHAP values, so a tree in $\Fe^{x_0}$ cannot fit these points in general, whereas one in $\Fp^{x_0}$ can.
\hfill\qedsymbol

We see that although SHAP explanations depend on a larger part of the data space and not just a local neighborhood around $x_0$, they are still uninformative for deep trees, even if the explained function has a GAM structure with trees as components. Only when restricting the depth of the trees, SHAP explanations become informative. 

\section{Anchor explanations}
\label{sec-anchors}

We now turn our attention to anchors, which are typically used in a classification setting. Hence, we assume in this section that $f$ is a scoring function with values in $[-1,1]$, from which the classification results are derived as $c_f(x) = \sign(f(x))$.  As is common, the anchor explanations are derived from the classification labels $c_f$, not from the scoring functions $f$. 
\begin{theorem}[\textbf{Anchors are only informative in case of simple functions or perfect precision}]\label{thm-anchors}
    Consider $\Xcal = \R^d$ with a probability distribution $\Pr$ that has a positive density.
    \begin{enumerate}
        \item Anchors $\Exanchor(\_,\_,p)$ with precision $0<p<1$ are fully uninformative for the class $\Tcal_\infty^{[-1,1]}$ of all trees with values in $[-1,1]$ and arbitrary depth.
        \item Anchors $\Exanchor(\_,\_,p)$ with precision $0\le p\le1$ are informative for the class $\Tcal_K^{[-1,1]}$ of all trees with values in $[-1,1]$ and depth at most $K$.
        \item Anchors $\Exanchor(\_,\_,p)$ with precision $p=1$ are informative for the class $\Tcal_\infty^{[-1,1]}$ of all trees with values in $[-1,1]$ and arbitrary depth.
    \end{enumerate} 
\end{theorem}

\textbf{Proof sketch} (full proof in Appendix \ref{app:proof-anchors}). For Part 1, consider a sample $x_1,...,x_n$. Arbitrary deep trees can have arbitrarily many and arbitrarily small leaves. This means, given a sample $x_1,...,x_n$ and some $g\in\Fp^{x_0}$, there is a $\tilde{g}$ which coincides with $g$ on the sample while it can still achieve a precision $\precision(R,g,x_0)$ equal to $p$, so $\tilde{g}\in\Fe^{x_0}$. Hence, the explanation is uninformative. \\
For Part 2, the proof works similarly to the one in the second part of Theorem \ref{thm-shap-trees}. As soon as the size of the sample $S$ and $x_0$ exceeds $2^K$, a tree needs all of its leaves to fit the points and hence has no further degrees of freedom to match an anchor with given precision. Consequently, $\Fe^{x_0}|_S\varsubsetneqq\Fp^{x_0}|_S$. \\
If the anchor is perfect, we get a region with positive probability mass, in which we know the sign of $f$. This restricts the prediction behavior of $\Fe^{x_0}|_S$ whenever parts of the sample $S$ lie in that region.\hfill\qedsymbol

At first, one might expect anchor explanations to be generally informative as they deliver a box with positive probability mass in which the output label is known. However, Part~1 of the theorem proves that this is not the case. The crucial point is that if the anchor is not perfect, a function can still produce any behavior while delivering the desired precision for the given anchor. This only changes if we restrict the function class as in Part 2 of the theorem. This fits perfectly into the image already painted by gradient explanations and SHAP explanations. The situation changes if our anchor is perfect. Then we indeed get a region with positive probability in which the classification output is perfectly known, resulting in informativeness, as shown in Part 3 of the theorem.

\section{Counterfactual explanations}

Counterfactual explanations are typically applied in a classification setting. We assume that $f$ is a scoring function with values in $[-1,1]$, from which the classification results are derived by thresholding: $c_f(x) = \sign(f(x))$.  As is common, counterfactual explanations are derived from the classification labels $c_f$, not from the scoring functions $f$. 
Importantly, we need to distinguish between weak and strong counterfactuals. 
\begin{theorem}[\textbf{Counterfactuals are only informative if strong or applied to simple functions}]\label{thm-counterfactuals}
Consider $\Xcal = \R^d$ with a probability distribution $\Pr$ that has a positive density.
\begin{enumerate}
    \item Weak counterfactuals $\Exwcf$ are fully uninformative for arbitrarily deep trees $\Tcal_\infty^{[-1,1]}$, continuous functions $\Ccal^{[-1,1]}$, and differentiable functions $\Dcal^{[-1,1]}$, each with values in $[-1,1]$.
    \item Weak counterfactuals $\Exwcf$ are informative for the class $\Tcal_K^{[-1,1]}$ of trees with values in $[-1,1]$ and depth at most $K$. They are also informative on the class $\Ccal_L$ of $L$-Lipschitz-continuous functions if additionally the scoring value $f(x_C)$ of the counterfactual $x_C$ is reported. 
    \item Strong counterfactuals $\Exscf$ are informative for arbitrarily deep trees $\Tcal_\infty^{[-1,1]}$, continuous functions $\Ccal^{[-1,1]}$, and differentiable functions $\Dcal^{[-1,1]}$, each with values in $[-1,1]$.
\end{enumerate}
\end{theorem}
\textbf{Proof sketch} (full proof in Appendix \ref{app:proof-cf}). 
The function classes in Part 1 can interpolate any finite sample of points. Given a sample $x_1,...,x_n$ and some $g\in\Fp^{x_0}$, we can find a function $\tilde{g}$ fitting $(x_1,g(x_1)),..., (x_n, g(x_n)), (x_0, f(x_0)), (x_C, f(x_C))$, where $x_C$ is the counterfactual. That means, $\tilde{g}\in\Fe^{x_0}$, implying that the explanation is uninformative.\\
Finite trees can only take $2^K$ many values. So, a function from $\Fp^{x_0}$ can fit $2^K-1$ points (needing one leaf for $x_0$), whereas a function from $\Fe^{x_0}$ can only fit $2^K-2$ points, needing another leaf for the counterfactual. This results in an informative explanation. Lipschitz-continuous functions can only change moderately fast, so in a neighborhood around the counterfactual, all points have the same label as the counterfactual, restricting the prediction behavior of $\Fe^{x_0}$ more than that of $\Fp^{x_0}$ in this region. \\
Similarly, in the case of a strong counterfactual, all points that are closer to $x_0$ than our counterfactual have the same label as $x_0$, again restricting the prediction behavior of $\Fe^{x_0}$ more than that of $\Fp^{x_0}$. \hfill\qedsymbol

Weak counterfactuals are, like the other explanation methods so far, only informative for small function classes. But how can this be? Even weak counterfactual explanations seem to intuitively give us a lot of information. The reason is that while the weak explanation tells us the label at one counterfactual data point, this does not allow us to make any conclusions about any other points in the space. The function could do pretty much anything and still be consistent with the explanation. This is different for strong counterfactuals, as shown by the last part of the theorem. Strong counterfactuals are always informative as they deliver not only a point but a whole region around $x_0$ with the same label (similarly to perfect anchors, see 
Section \ref{sec-anchors}). However, in all practical implementations, counterfactuals are not strong: they are based on search heuristics to find some close point with a different label,  which then lacks the guarantee of the distance being minimal.

\section{The practical problem with uninformative explanations}
\label{sec-connection-to-real-world}

\begin{figure*}[t]
     \centering
      \begin{subfigure}[b]{0.3\textwidth}
          \includegraphics[width=\textwidth]{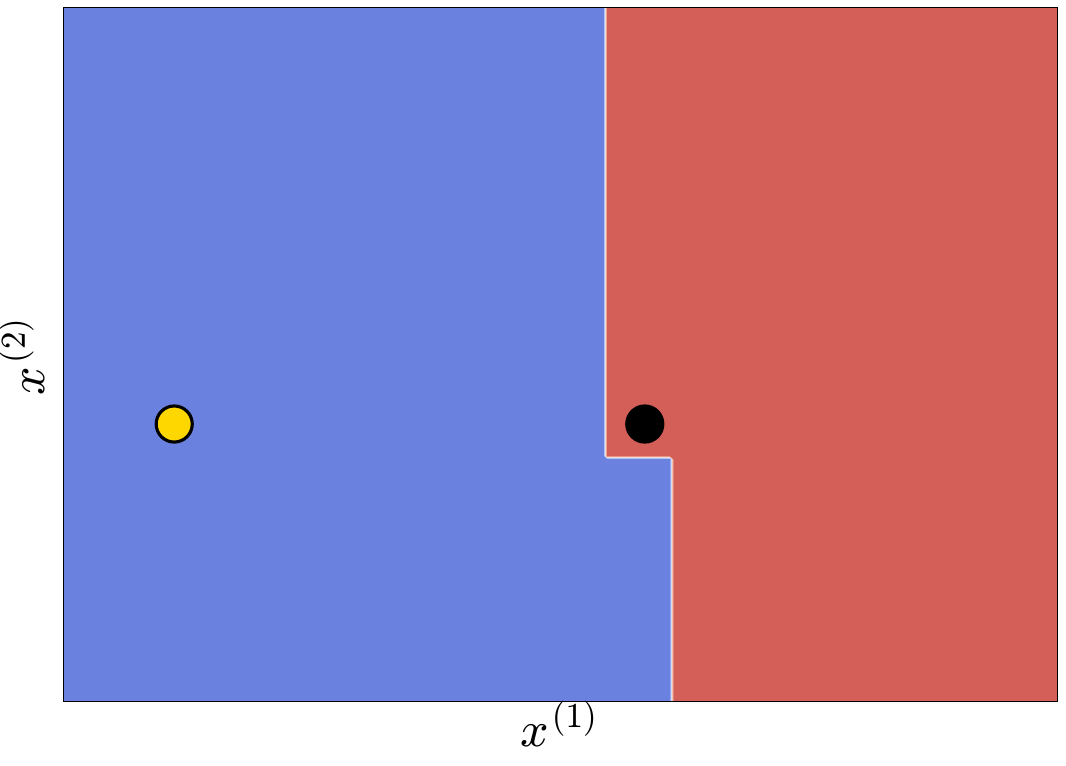}
      \end{subfigure}
           \begin{subfigure}[b]{0.3\textwidth}
          \includegraphics[width=\textwidth]{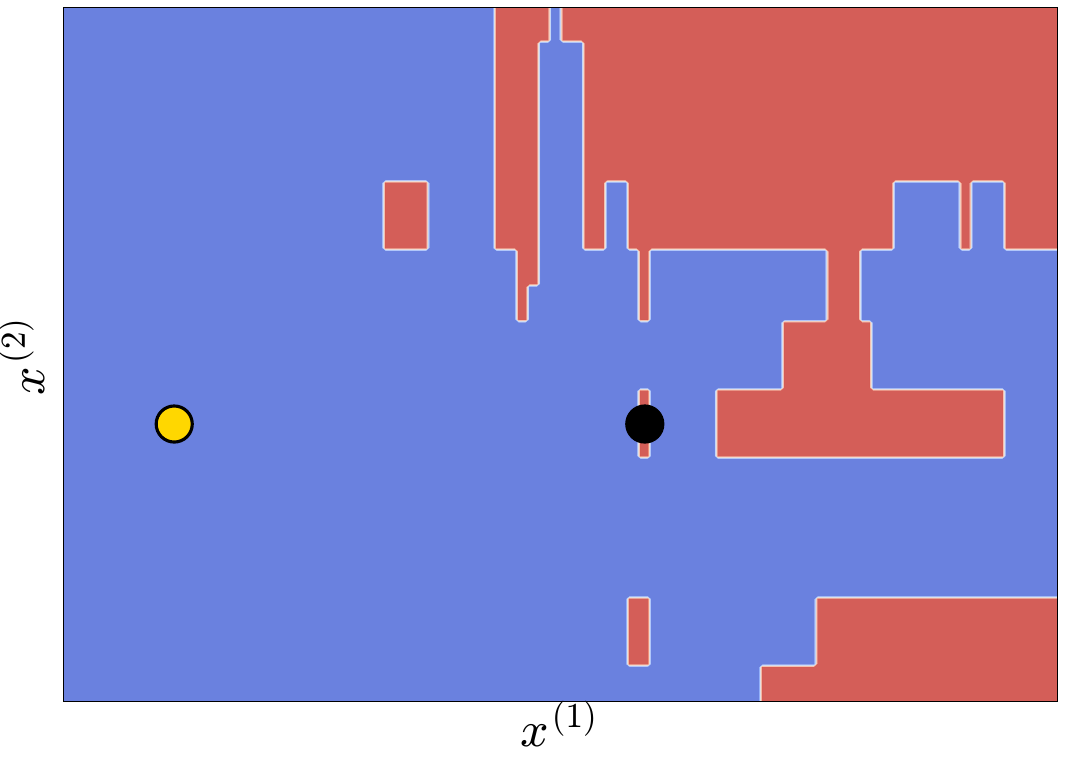}
      \end{subfigure}
 \hspace{1cm}%
        \caption{The two plots each show a data point (yellow) and its counterfactual (black). The left plot illustrates a decision boundary of a shallow tree, as it might implicitly be assumed by the user who tries to interpret the explanation, for example, in order to challenge it, act on it, or place trust in it. Due to the simple decision boundary, the label of the counterfactual extends to a neighborhood.
        The right panel shows the decision boundary of a different model that is also consistent with the same explanation; for this model, the recipient's intuitive interpretations might be invalid, as the counterfactual's stability region is tiny. Our framework formalizes this intuition: Across a wide range of explanation algorithms, an explanation can be informative only if the recipient knows that the decision function is ``sufficiently simple''.
        }
        \label{fig:illustration}
\end{figure*}

So far, we have shown that explanations are often uninformative in a mathematical sense. In this section, we argue why we believe that uninformative explanations are not useful in practice.

\paragraph{Working with explanations involves drawing implicit conclusions about the model's behavior.} 
It has been argued many times in the literature that upon receiving the explanation, a user tends to use the explanations to derive conjectures about the underlying decision function, for example, its behavior on other points 
\citep{miller2019explanation,jacovi2021formalizing,mohseni2021multidisciplinary,jalali2023predictability,
liao2021human,BorRaiLux25}. 
A SHAP explanation invites a doctor to read the per-feature contributions as a guide to the patient's trajectory if the highlighted feature were acted upon, which is a claim about the model's behavior at points related to but different from $x_0$. 
A recipient of a gradient explanation might take the explanation as evidence for how moderate perturbations of the point $x_0$ would change the outcome, which is a claim about the model's behavior on a ball around $x_0$.  
A counterfactual explanation prompts the recipient to act
on it, often under the implicit assumption that the dependence reported on one feature still holds under small variations of others.  
In each case, the explanation of the model at $x_0$ is being used to derive conjectures addressing the prediction behavior of our function, and without further assumptions, none of these conjectures needs to be true. 
While this has been pointed out repeatedly for particular conjectures derived from particular explanations under particular assumptions (e.g., \citealp{ghorbani2019interpretation} for gradients, \citealp{Kumar20,huang2024failings} for SHAP, \citealp{barocas2020hidden} for counterfactuals, and \citealp{bilodeau2024impossibility} for feature attributions in general), our framework now treats the issue in much larger generality. Our definition of an informative explanation tracks whether, by exploiting knowledge about the explanation, the recipient could possibly make any reliable conclusion regarding the function's behavior that they could not already derive from the prediction alone. For an uninformative explanation, this is not the case: No matter which dataset a user considers, with probability 1 over the choice of this dataset, the explanation provides no additional information about the function's behavior on it. 
In this sense, we consider the informativeness of an explanation as a necessary (but by no means sufficient!) criterion that any explanation algorithm should satisfy.

Throughout the paper, it has been shown that uninformativeness tends to occur when the underlying function class is very rich. Figure \ref{fig:illustration} illustrates this phenomenon as well as the problem with uninformative explanations: In both plots, the counterfactual (black dot) represents a single point with a different label than the input point (yellow dot). However, if the decision boundary of a function is not restricted, because $\Fcal$ is very complex, its prediction behavior is arbitrary in the neighborhood of the counterfactual. This is illustrated in the right panel, where we see the decision boundary of a random forest, whose predictions deviate from the counterfactual as soon as we move away from it just a little bit. The counterfactual carries no information about $f$ beyond the label of a single point. This might cause severe misinterpretations in practice. 
To be able to draw trustworthy conclusions, one would need to be able to rely on additional know\-ledge about the decision function. For example, the fact that the decision tree is shallow, like in the left panel of Figure \ref{fig:illustration}, might guarantee that the decision function cannot change arbitrarily fast, so further knowledge about the local behavior of the decision function can be inferred from the counterfactual. 

\section{Extending the framework}
In the following, we discuss two possible variants of the framework. First, we show how to generalize the notion of informativeness to multiple explanations. Afterward, we discuss an alternative definition that uses Rademacher complexity. 

\subsection{Extending the framework to multiple explanations}
\label{sec-multiple-explanations}

The framework derived in the last sections extends seamlessly to the case where we receive predictions and explanations for several data points. Given a set of points $T= \{z_1, \dots, z_m\}$, we define 
\ba
& \Fp^T := \{ g \in \Fcal \,|\, \forall z \in T: g(z) = f(z)\},\\
& \Fe^T := \{ g \in \Fcal \,|\, \forall z \in T: g(z) = f(z),\ex(g,z) = \ex (f, z)\}
\ea
for unique explanations and 
$$\Fe^T := \{ g \in \Fcal \,|\, \forall z \in T: g(z) = f(z),\ex(f,z) \in \Ex (g, z)\}$$
for a non-unique explanation algorithm and given explanations $\ex(f,z_1),...,\ex(f,z_m)$. Analogously to before we call the set of explanations {non-trivial} if $\Fe^T \varsubsetneqq \Fp^T$ and informative with respect to $n\in\Nat$ if $\mathrm{Pr}_{S\sim \Pr^n}\left(\Fe^T|_S \varsubsetneqq \Fp^T|_S\right)>0$. 
Just like in the case of one explanation, we call an explanation algorithm $\ex$ (resp. $\Ex$) informative for $\Fcal$ and $m\in\Nat$ explanations if for all $f\in\Fcal$ and with probability 1 over the choice of points $T=\{z_1,...,z_m\}$ there exists an $N\in\Nat$ such that the explanations $\ex(f,z_1),..., \ex(f,z_m)$ (resp. all $(\ex(f,z_1), ..., \ex(f,z_m))\in\prod_{j=1}^m\Ex(f,z_j)$) are informative for all $n\ge N$. We call an explanation algorithm fully uninformative for $\Fcal$ and $m\in\Nat$ explanations if for all $f\in\Fcal$ and with probability 1 over the choice of points $T=\{z_1,...,z_m\}$, the explanations $\ex(f,z_1),..., \ex(f,z_m)$ (resp. all $(\ex(f,z_1), ..., \ex(f,z_m))\in\prod_{j=1}^m\Ex(f,z_j)$) are uninformative for all $n\in\Nat$. 
Although several explanations seem to reveal more information than a single one, our negative results carry over:

\begin{theorem}[\textbf{Multiple explanations remain uninformative for complex function spaces}]\label{thm-multiple-explanations}
The statements of Theorem \ref{thm-gradients} (Part 1), Theorem \ref{thm-shap-trees} (Part 1 \& 2), Theorem \ref{thm-anchors} (Part 1, for non-overlapping anchors), and Theorem \ref{thm-counterfactuals} (Part 1) remain valid for finitely many explanations with respect to the above definition. That is, if gradient explanations, SHAP explanations, anchor explanations, or counterfactual explanations are fully uninformative for one explanation, they are fully uninformative for $m$ explanations as well, where $m\in\Nat$. 
\end{theorem}

\textbf{Proof sketch.} The techniques in our proofs can also be applied to finitely many points and explanations. For more details, see Appendix \ref{app:proof-several-explanations}.\hfill\qedsymbol

This result might be surprising at first glance. One might think that observing many local explanations should even allow us to gain a global understanding of the underlying function, and this is also how local explanations are sometimes used in practice, for example, by computing average SHAP values \citep{bhattacharjee2025safely}. 
To understand why our theorem is true, observe that in our framework, $\Fp^T\setminus \Fe^T$ represents which functions we can {\em safely} exclude after receiving explanations, similar to a version space argument in learning theory. In the setting of our theorem, safely excluding functions is still not possible after seeing many explanations: Each explanation $\ex(f,z_j)$ can be satisfied by modifying $f$ on an arbitrarily small part of the space, so that rich function classes have enough flexibility to interpolate any finite sample. In the long run, it would be interesting to 
extend our framework to a scenario where, after seeing many explanations, 
a user can rule out functions based on high-probability arguments rather than with absolute certainty.  We leave this for future research.

\subsection{Alternative definition using Rademacher complexities}
\label{sec:alternative definition rademacher}
One alternative definition of informativeness, used in a prior version of the paper, uses the Rademacher complexity, which quantifies the complexity of a function class $\Fcal$ by its ability to fit random labels \citep{bartlett2002rademacher,bousquet2003introduction,shalev2014understanding}. 
Concretely, for $n \in \Nat$ denote by $\sigma:= (\sigma_1, \sigma_2, ..., \sigma_n)$ a vector of $n$ i.i.d. random variables $\sigma_i$ that take the values $-1$ and $+1$ with probability 1/2 each. For a fixed set of points $S = (x_1, ..., x_n) \in \Xcal^n$ we define the empirical Rademacher complexity by
$$
\hat R_n(\Fcal) := 
\E_\sigma \left(\sup_{f \in \Fcal} 
\frac{1}{n} 
\sum_{i=1}^n
\sigma_i f(x_i)\right)
$$
and the Rademacher complexity
$$R_n(\Fcal) := \E_{S\sim \Pr^n} \left(\hat R_n(\Fcal)\right) 
$$ 
as its expectation over the random sample. 
We now call an explanation $\ex(f,x_0)$ Rademacher-informative with respect to $n\in\Nat$ if the Rademacher complexity of $\Fe^{x_0}$ is smaller than the one of $\Fp^{x_0}$:
$$R_n(\Fe^{x_0})<R_n(\Fp^{x_0}).$$
As the Rademacher complexity is designed to measure how well a function class can fit arbitrary labels, the notion of Rademacher-informativeness is closely related to our notion of informativeness, although they are not equivalent. {\bf Rademacher-informativeness implies informativeness}, which can easily be verified by proving the contraposition: If it holds \linebreak $\mathrm{Pr}_{S\sim \Pr^n}\left(\Fe^{x_0}|_S =\Fp^{x_0}|_S\right)=1$, this directly implies 
$$\sup_{f \in \Fe^{x_0}} 
\frac{1}{n} 
\sum_{i=1}^n
\sigma_i f(x_i) = \sup_{f \in \Fp^{x_0}} 
\frac{1}{n} 
\sum_{i=1}^n
\sigma_i f(x_i)$$
for every sample of $\sigma$ and with probability 1 over the choice of the sample $S=(x_1,...,x_n)$. In particular, taking expectations over $\sigma$ and $S$ preserves the equality, so $R_n(\Fe^{x_0})=R_n(\Fp^{x_0})$. 

The other way round does not always hold: in general, it is not true that explanations that are informative are also Rademacher-informative. However, all the positive results about informativeness in this paper hold with respect to Rademacher-informativeness as well, as shown in the proofs (Theorem \ref{thm-shap-trees}, Part 3 needs restriction to bounded trees to ensure finite Rademacher complexity).

We decided to work with the Definition presented in Section \ref{Definition Informativeness} because it is more intuitive to interpret than Rademacher-informativeness, and it makes our negative results even stronger.

\section{Discussion and outlook}

We introduce a learning-theoretic notion of informative explanations for tabular data. Our framework shows that widely used methods (gradients, SHAP, counterfactuals, anchors) fail to be informative for complex function classes. They only become informative when the complexity of the function is bounded and is communicated to the explanation recipient as well. 

\paragraph{The scope of our negative results.}
Our negative results hold in a wide range of scenarios: They hold not only in an adversarial scenario \citep{BorFinRaiLux22} where the explanation provider tries to cheat or to manipulate explanations, but also in a cooperative setting where the explanation provider acts with the best intentions. 
Our results hold regardless of human usability issues  
\citep{liao2021human,BorRaiLux25}: Missing information cannot be ``fixed'' in hindsight by the way information is extracted from an explanation or how the explanation is presented to the user. 
While we proved our results for exactly computed explanations (based on true expectations, ignoring computational issues), approximated explanations inherit the same limitation: If a perfectly computed explanation contains no information, neither does its estimate. 
And finally, our results can also be extended to a multi-explanation scenario. 

\paragraph{Uninformative explanations in practice.}
Do our results mean that local post-hoc explanations are wrong or useless? No. 
There are many scenarios where explanations serve as a useful heuristic, like in exploratory data analysis or code debugging. If some explanations are uninformative or occasionally lead to unjustified conclusions, this may not be a serious problem in these cases. However, our results pose a challenge in high-risk applications, where users 
expect explanations that are reliable and not misleading. We and many others argue that in such scenarios, the recipients want to draw conclusions about the prediction behavior of the underlying model, in particular, on other data points \citep{miller2019explanation,jacovi2021formalizing,mohseni2021multidisciplinary,jalali2023predictability,
liao2021human,BorRaiLux25}, as discussed in Section \ref{sec-connection-to-real-world}. And such conclusions cannot be drawn reliably if the underlying explanation is not informative.

\paragraph{Informative explanations in practice.} 
Our results show from a formal point of view how explanations can be made informative. Rather than conveying only the raw explanation, an explanation needs to be augmented with additional information, for example, about the complexity of the decision function. 
But is an informative explanation then always useful and fit for purpose in practice? Again, the answer is no. We consider informativeness as a necessary, but by no means sufficient criterion for a good and helpful explanation. 
This opens promising directions for future research.

\paragraph{Future research: how much information?}  
A natural next step would be to move from \emph{whether} an explanation carries information to {quantifying \emph{how much} information it carries.} In particular, our current framework leaves open the question of how to {\em use} or {\em extract} helpful information from explanations, and {how much information then would be needed in practice to draw concrete, actionable insights from explanations in a specific application scenario. } This also comes with the question of how to present the information in a user-friendly way. We consider the symbiosis between deriving theoretical guarantees for explanation algorithms and human-computer interaction an important topic of future research.

\paragraph{Limitations.} Our approach considers the setup of local post-hoc explanations for tabular data. We leave an extension of the framework to other modalities, like text data and LLMs, for future research. We study the informativeness of explanations from a theoretical point of view, pointing out that the informativeness of explanations should be a necessary, but by no means sufficient, criterion for explanations. This work provides a general framework and does not discuss specific use cases or usability issues of explanations. 

\subsection*{Generative AI usage statement}
All ideas, theorems, proofs, and references were created by us without support of generative AI. We wrote the entire paper ourselves, but we have used AI tools to find typos or correct grammar and wording.

\subsection*{Acknowledgements}

This work has been supported by the German Research Foundation through
the Cluster of Excellence “Machine Learning - New Perspectives for
Science" (EXC 2064/1 number 390727645) and Project
560788681, the Carl Zeiss Foundation
through the CZS Center for AI and Law, and the International Max
Planck Research School for Intelligent Systems (IMPRS-IS).
Balázs Szabados was supported by the ``Robust Uncertainty Quantification for Learning and Control'' ADVANCED research project of the National Research, Development and Innovation Office of Hungary (NKFIH), grant number 153390.

\bibliography{references}

@inproceedings{Freiesleben_Meding_König_2026, 
title={Explainable AI Isn’t Enough! Rethinking Algorithmic Contestability},  booktitle={Conference on AI, Ethics, and Society ({AIES})}, 
author={Freiesleben, Timo and Meding, Kristof and König, Gunnar}, 
year={2026}}

@article{hwang2026explanation,
  title={Explanation Multiplicity in {SHAP}: Characterization and Assessment},
  author={Hwang, Hyunseung and Lee, Seungeun and Rosenblatt, Lucas and Whang, Steven Euijong and Stoyanovich, Julia},
  journal={arXiv preprint arXiv:2601.12654},
  year={2026}
}

@inproceedings{tekkesinoglu2026explanations,
  title={When Explanations Deceive: Understanding Unintentional and Intentional Deception in XAI},
  author={Tekkesinoglu, Sule},
  booktitle={Conference on Fairness, Accountability, and Transparency (FAccT)},
  year={2026}
}

@inproceedings{fresz2026explainable,
  title={Explainable AI for the EU right to explanation: A systematic review of the law-XAI translation gap},
  author={Fresz, Benjamin and Dubovitskaya, Elena and Huber, Marco F},
  booktitle={Conference on AI, Ethics, and Society ({AIES})},
  year={2026}
}

@inproceedings{jalali2023predictability,
  title={Predictability and comprehensibility in post-hoc xai methods: A user-centered analysis},
  author={Jalali, Anahid and Haslhofer, Bernhard and Kriglstein, Simone and Rauber, Andreas},
  booktitle={Science and Information Conference},
  year={2023}
}

@article{mohseni2021multidisciplinary,
  title={A multidisciplinary survey and framework for design and evaluation of explainable AI systems},
  author={Mohseni, Sina and Zarei, Niloofar and Ragan, Eric D},
  journal={ACM Transactions on Interactive Intelligent Systems (TiiS)},
  volume={11},
  number={3-4},
  pages={1--45},
  year={2021}
}

@inproceedings{jacovi2021formalizing,
  title={Formalizing trust in artificial intelligence: Prerequisites, causes and goals of human trust in AI},
  author={Jacovi, Alon and Marasovi{\'c}, Ana and Miller, Tim and Goldberg, Yoav},
  booktitle={Conference on fairness, accountability, and transparency (FAccT)},
  year={2021}
}

@inproceedings{bhattacharjee2025safely,
  title={How to safely discard features based on aggregate SHAP values},
  author={Bhattacharjee, Robi and Frohnapfel, Karolin and von Luxburg, Ulrike},
  booktitle={Conference on Learning Theory (COLT)},
  year={2025},
}

@inproceedings{rudin2024amazing,
  title={Amazing things come from having many good models},
  author={Rudin, Cynthia and Zhong, Chudi and Semenova, Lesia and Seltzer, Margo and Parr, Ronald and Liu, Jiachang and Katta, Srikar and Donnelly, Jon and Chen, Harry and Boner, Zachery},
  booktitle={International Conference on Machine Learning (ICML)},
  year={2024}
}

@inproceedings{Black_2022_ModelMultiplicity,
    title={Model multiplicity: Opportunities, concerns, and solutions},
    author={Black, Emily and Raghavan, Manish and Barocas, Solon},
    booktitle={Conference on Fairness, Accountability, and Transparency (FAccT)},
    year={2022}
}

@inproceedings{Marx_2020_PredictiveMultiplicity,
    title={Predictive multiplicity in classification},
    author={Marx, Charles and Calmon, Flavio and Ustun, Berk},
    year={2020},
    booktitle={Proceedings of the International Conference on Machine Learning (ICML)},
}

@article{Breiman_2001_StatisticalModeling,
  title={Statistical modeling: The two cultures (with comments and a rejoinder by the author)},
  author={Breiman, Leo},
  journal={Statistical Science},
  volume={16},
  number={3},
  pages={199--231},
  year={2001},
}

@article{bartlett2002rademacher,
  title={{R}ademacher and {G}aussian complexities: Risk bounds and structural results},
  author={Bartlett, Peter and Mendelson, Shahar},
  journal={JMLR},
  volume={3},
  pages={463--482},
  year={2002}
}

@book{molnar2025,
  title     = {Interpretable Machine Learning},
  subtitle  = {A Guide for Making Black Box Models Explainable},
  author    = {Christoph Molnar},
  year      = {2025},
  edition   = {3rd},
publisher = {Leanpub},
}

@article{miller2019explanation,
  title={Explanation in artificial intelligence: Insights from the social sciences},
  author={Miller, Tim},
  journal={Artificial intelligence},
  volume={267},
  pages={1--38},
  year={2019},
  publisher={Elsevier}
}

@article{liao2021human,
  title={Human-centered explainable {AI} ({XAI}): From algorithms to user experiences},
  author={Liao, Q Vera and Varshney, Kush R},
  journal={arXiv preprint arXiv:2110.10790},
  year={2021}
}

@inproceedings{panigutti2023role,
  title={The role of explainable {AI} in the context of the {AI} {A}ct},
  author={Panigutti, Cecilia and Hamon, Ronan and Hupont, Isabelle and Fernandez Llorca, David and Fano Yela, Delia and Junklewitz, Henrik and Scalzo, Salvatore and Mazzini, Gabriele and Sanchez, Ignacio and Soler Garrido, Josep and Gomez, Emilia},
  booktitle={Conference on Fairness, Accountability, and Transparency (FAccT)},
  year={2023}
}

@article{karimi2022survey,
  title={A survey of algorithmic recourse: contrastive explanations and consequential recommendations},
  author={Karimi, Amir-Hossein and Barthe, Gilles and Sch{\"o}lkopf, Bernhard and Valera, Isabel},
  journal={ACM Computing Surveys},
  volume={55},
  number={5},
  pages={1--29},
  year={2022},
  publisher={ACM New York, NY}
}

@inproceedings{velmurugan2021evaluatingfidelity,
  title={Evaluating fidelity of explainable methods for predictive process analytics},
  author={Velmurugan, Mythreyi and Ouyang, Chun and Moreira, Catarina and Sindhgatta, Renuka},
  booktitle={International Conference on Advanced Information Systems Engineering},
  year={2021}
}

@inproceedings{adebayo2020debugging,
  title={Debugging Tests for Model Explanations},
  author={Adebayo, Julius and Muelly, Michael and Liccardi, Ilaria and Kim, Been},
  booktitle={Neural Information Processing Systems (NeurIPS)},
  year={2020}
}

@article{agarwal2018automated,
  title={Automated test generation to detect individual discrimination in {AI} models},
  author={Agarwal, Aniya and Lohia, Pranay and Nagar, Seema and Dey, Kuntal and Saha, Diptikalyan},
  journal={arXiv preprint arXiv:1809.03260},
  year={2018}
}

@inproceedings{Kumar20,
  author       = {I. Elizabeth Kumar and
                  Suresh Venkatasubramanian and
                  Carlos Scheidegger and
                  Sorelle A. Friedler},
  title        = {Problems with {S}hapley-value-based explanations as feature importance
                  measures},
  booktitle    = {International Conference on Machine Learning (ICML)},
  year         = {2020},
}

@inproceedings{alvarez2018towards,
  title={Towards robust interpretability with self-explaining neural networks},
  author={Alvarez Melis, David and Jaakkola, Tommi},
  booktitle={Neural Information Processing Systems (NeurIPS)},
  year={2018}
}

@inproceedings{bressan2024theory,
  title={A theory of interpretable approximations},
  author={Bressan, Marco and Cesa-Bianchi, Nicol{\`o} and Esposito, Emmanuel and Mansour, Yishay and Moran, Shay and Thiessen, Maximilian},
  booktitle={Conference on Learning Theory (COLT)},
  year={2024},
}

@inproceedings{molnar2020general,
  title={General pitfalls of model-agnostic interpretation methods for machine learning models},
  author={Molnar, Christoph and K{\"o}nig, Gunnar and Herbinger, Julia and Freiesleben, Timo and Dandl, Susanne and Scholbeck, Christian A and Casalicchio, Giuseppe and Grosse-Wentrup, Moritz and Bischl, Bernd},
  booktitle={International Workshop on Extending Explainable AI Beyond Deep Models and Classifiers},
  year={2020}
}

@article{krishnan2020against,
  title={Against interpretability: a critical examination of the interpretability problem in machine learning},
  author={Krishnan, Maya},
  journal={Philosophy \& Technology},
  volume={33},
  number={3},
  pages={487--502},
  year={2020},
  publisher={Springer}
}

@article{lipton2018mythos,
  title={The mythos of model interpretability: In machine learning, the concept of interpretability is both important and slippery.},
  author={Lipton, Zachary C},
  journal={Queue},
  volume={16},
  number={3},
  pages={31--57},
  year={2018},
  publisher={ACM New York, NY, USA}
}

@inproceedings{bhatt2021evaluating,
  title={Evaluating and aggregating feature-based model explanations},
  author={Bhatt, Umang and Weller, Adrian and Moura, Jos{\'e} MF},
  booktitle={International Joint Conference on Artificial Intelligence (IJCAI)},
  year={2021}
}

@inproceedings{enouen2025instashap,
  title={Insta{SHAP}: Interpretable additive models explain Shapley values instantly},
  author={Enouen, James and Liu, Yan},
  booktitle={International Conference on Learning Representations ({ICLR})},
  year={2025}
}

@article{rudin2019stop,
  title={Stop explaining black box machine learning models for high stakes decisions and use interpretable models instead},
  author={Rudin, Cynthia},
  journal={Nature Machine Intelligence},
  volume={1},
  number={5},
  pages={206--215},
  year={2019},
}

@InProceedings{hullman2025explanationsmeansend,
  title = 	 {Explanations are a Means to an End: Decision Theoretic Explanation Evaluation},
  author =       {Guo, Ziyang and Ustun, Berk and Hullman, Jessica},
  booktitle = 	 {International Conference on Machine Learning ({ICML})},
  year = 	 {2026}
}

@inproceedings{sundararajan2017axiomatic,
  title     = {Axiomatic Attribution for Deep Networks},
  author    = {Mukund Sundararajan and Ankur Taly and Qiqi Yan},
  booktitle = {International Conference on Machine Learning (ICML)},
  year      = {2017}
}

@article{smilkov2017smoothgrad,
  title        = {{SmoothGrad}: removing noise by adding noise},
  author       = {Daniel Smilkov and Nikhil Thorat and Been Kim and Fernanda Viégas and Martin Wattenberg},
  year         = {2017},
journal = {arxiv preprint arxiv:1706.03825}
}

@article{WachterEtal17,
  title={Counterfactual explanations without opening the black box: Automated decisions and the {GDPR}},
  author={Wachter, Sandra and Mittelstadt, Brent and Russell, Chris},
  journal={Harvard Journal of Law \& Technology},
  volume={31},
  number={2},
  pages={841},
  year={2017},
}

@inproceedings{BorRaiLux25, 
  title={Rethinking Explainable Machine Learning as Applied Statistics},
  author={Bordt, Sebastian and Raidl, Eric and {von~Luxburg}, Ulrike},
  booktitle={International Conference on Machine Learning (ICML)},
  year = {2025}
}

@inproceedings{slack2021counterfactual,
  title={Counterfactual explanations can be manipulated},
  author={Slack, Dylan and Hilgard, Anna and Lakkaraju, Himabindu and Singh, Sameer},
  booktitle={Neural Information Processing Systems (NeurIPS)},
  year={2021}
}

@article{sharma2024x,
  title={X Hacking: The Threat of Misguided {AutoML}},
  author={Sharma, Rahul and Redyuk, Sergey and Mukherjee, Sumantrak and Sipka, Andrea and Vollmer, Sebastian and Selby, David},
  journal={arXiv preprint arXiv:2401.08513},
  year={2024}
}

@article{kaminski2025right,
  title={The Right to Explanation in the {AI} {A}ct},
  author={Kaminski, Margot E and Malgieri, Gianclaudio},
  journal={U of Colorado Law Legal Studies Research Paper No. 25-9, available at SSRN 5194301},
  year={2025}
}

@inproceedings{nannini2024habemus,
  title={Habemus a Right to an Explanation: so {W}hat?--{A} Framework on Transparency-Explainability Functionality and Tensions in the {EU AI Act}},
  author={Nannini, Luca},
  booktitle={Conference on {AI}, {E}thics, and {S}ociety (AIES)},
  year={2024}
}

@incollection{bousquet2003introduction,
  title={Introduction to statistical learning theory},
  author={Bousquet, Olivier and Boucheron, St{\'e}phane and Lugosi, G{\'a}bor},
  booktitle={Advanced Lectures on Machine Learning, MLSS 2003},
  pages={169--207},
  year={2003},
  publisher={Springer}
}

@book{shalev2014understanding,
  title={Understanding machine learning: From theory to algorithms},
  author={Shalev-Shwartz, Shai and Ben-David, Shai},
  year={2014},
  publisher={Cambridge University Press}
}

@inproceedings{Ribeiro18,
  author       = {Marco T{\'{u}}lio Ribeiro and
                  Sameer Singh and
                  Carlos Guestrin},
  title        = {Anchors: High-Precision Model-Agnostic Explanations},
  booktitle    = {{AAAI} Conference on Artificial Intelligence},
  year         = {2018},
}

@inproceedings{Ribeiro16,
  author       = {Marco T{\'{u}}lio Ribeiro and
                  Sameer Singh and
                  Carlos Guestrin},
  title        = {``{W}hy Should {I} Trust You?": Explaining the Predictions of Any Classifier},
  booktitle    = {International Conference on Knowledge Discovery and Data Mining (KDD)},
  year         = {2016},
}

@article{huang2024failings,
  title={On the failings of Shapley values for explainability},
  author={Huang, Xuanxiang and Marques-Silva, Joao},
  journal={International Journal of Approximate Reasoning},
  volume={171},
  pages={109112},
  year={2024}
}

@inproceedings{ghorbani2019interpretation,
  title={Interpretation of neural networks is fragile},
  author={Ghorbani, Amirata and Abid, Abubakar and Zou, James},
  booktitle={{AAAI} Conference on Artificial Intelligence},
  year={2019}
}

@inproceedings{barocas2020hidden,
  title={The hidden assumptions behind counterfactual explanations and principal reasons},
  author={Barocas, Solon and Selbst, Andrew D and Raghavan, Manish},
  booktitle={Conference on {F}airness, {A}ccountability, and {T}ransparency ({FAccT})},
  year={2020}
}

@inproceedings{Dasgupta22,
  author       = {Sanjoy Dasgupta and
                  Nave Frost and
                  Michal Moshkovitz},
  title        = {Framework for Evaluating Faithfulness of Local Explanations},
  booktitle    = {International Conference on Machine Learning ({ICML})},
  year         = {2022},
}

@inproceedings{Lundberg17,
  author       = {Scott M Lundberg and
                  Su{-}In Lee},
  title        = {A Unified Approach to Interpreting Model Predictions},
  booktitle    = {Neural Information Processing Systems ({NeurIPS})},
  year         = {2017}
}

@InProceedings{BorFinRaiLux22,
  author =       {Sebastian Bordt  and  Michèle Finck  and  Eric Raidl  and  Ulrike {von~Luxburg}},
  title =        {Post-Hoc Explanations Fail to Achieve their Purpose in Adversarial Contexts},
  booktitle = {Conference on Fairness, Accountability, and Transparency ({FAccT})},
  year =      2022}

@article{Nauta_2023,
   title={From Anecdotal Evidence to Quantitative Evaluation Methods: A Systematic Review on Evaluating Explainable {AI}},
   volume={55},
   number={13s},
   journal={ACM Computing Surveys},
   author={Nauta, Meike and Trienes, Jan and Pathak, Shreyasi and Nguyen, Elisa and Peters, Michelle and Schmitt, Yasmin and Schlötterer, Jörg and van Keulen, Maurice and Seifert, Christin},
   year={2023},
   pages={1–42} }

@article{bilodeau2024impossibility,
  title={Impossibility theorems for feature attribution},
  author={Bilodeau, Blair and Jaques, Natasha and Koh, Pang Wei and Kim, Been},
  journal={Proceedings of the National Academy of Sciences},
  volume={121},
  number={2},
  pages={e2304406120},
  year={2024},
  publisher={National Acad Sciences}
}

@inproceedings{anders2020fairwashing,
  title={Fairwashing explanations with off-manifold detergent},
  author={Anders, Christopher and Pasliev, Plamen and Dombrowski, Ann-Kathrin and M{\"u}ller, Klaus-Robert and Kessel, Pan},
  booktitle={International Conference on Machine Learning (ICML)},
  year={2020},
}

@article{Fokkema23,
  title={Attribution-based explanations that provide recourse cannot be robust},
  author={Fokkema, Hidde and De Heide, Rianne and Van Erven, Tim},
  journal={Journal of Machine Learning Research},
  volume={24},
  number={360},
  pages={1--37},
  year={2023}
}

@inproceedings{han2022explanation,
  title={Which explanation should {I} choose? {A} function approximation perspective to characterizing post hoc explanations},
  author={Han, Tessa and Srinivas, Suraj and Lakkaraju, Himabindu},
  booktitle={Neural Information Processing Systems (NeurIPS)},
  year={2022}
}

@inproceedings{slack2020fooling,
  title={Fooling {LIME} and {SHAP}: Adversarial attacks on post hoc explanation methods},
  author={Slack, Dylan and Hilgard, Sophie and Jia, Emily and Singh, Sameer and Lakkaraju, Himabindu},
  booktitle={Conference on AI, Ethics, and Society (AIES)},
  year={2020}
}

@inproceedings{adebayo2018sanity,
  title={Sanity checks for saliency maps},
  author={Adebayo, Julius and Gilmer, Justin and Muelly, Michael and Goodfellow, Ian and Hardt, Moritz and Kim, Been},
  booktitle={Neural Information Processing Systems (NeurIPS)},
  year={2018}
}

@InProceedings{azzolin2025beyond,
  title={Beyond Topological Self-Explainable {GNN}s: A Formal Explainability Perspective},
  author={Azzolin, Steve and Malhotra, Sagar and Passerini, Andrea and Teso, Stefano},
  booktitle={International Conference on Machine Learning {(ICML)}},
  year={2025}
}

@inproceedings{BhaLux2024,
  author =    {Robi Bhattacharjee and Ulrike {von~Luxburg}},
  title =     {Auditing local explanations is hard},
  booktitle = {Neural Information Processing Systems (NeurIPS)},
  year =      2024}

@inproceedings{WangTheory22,
  author       = {Zifan Wang and
                  Matt Fredrikson and
                  Anupam Datta},
  title        = {Robust Models Are More Interpretable Because Attributions Look Normal},
  booktitle    = {International Conference on Machine Learning {(ICML)}}, 
  year         = {2022},
}

@InProceedings{BorLux_shap_2023,
  author =       {Sebastian Bordt and Ulrike {von~Luxburg}},
  title =     {From {Shapley} Values to Generalized Additive Models and back},
  year =      2023, 
  booktitle = {Artificial Intelligence and Statistics (AISTATS)}}

@InProceedings{GarLux20,
  author =       {Damien Garreau and Ulrike {von~Luxburg}},
  title =        {Explaining the Explainer: A First Theoretical Analysis of {LIME}},
  booktitle = {Artificial Intelligence and Statistics (AISTATS)},
  year =      2020}

@inproceedings{HillMasoomi24,
  author       = {Davin Hill and
                  Aria Masoomi and
                  Max Torop and
                  Sandesh Ghimire and
                  Jennifer G. Dy},
  title        = {Boundary-Aware Uncertainty for Feature Attribution Explainers},
  booktitle    = {Artificial Intelligence and Statistics (AISTATS)},
  year         = {2024}
}

@inproceedings{caruana2015intelligible,
  title={Intelligible models for healthcare: Predicting pneumonia risk and hospital 30-day readmission},
  author={Caruana, Rich and Lou, Yin and Gehrke, Johannes and Koch, Paul and Sturm, Marc and Elhadad, Noemie},
  booktitle={International Conference on Knowledge Discovery and Data Mining (KDD)},
  year={2015}
}

@article{letham2015interpretable,
author = {Benjamin Letham and Cynthia Rudin and Tyler H. McCormick and David Madigan},
title = {{Interpretable classifiers using rules and Bayesian analysis: Building a better stroke prediction model}},
volume = {9},
journal = {The Annals of Applied Statistics},
number = {3},
pages = {1350 -- 1371},
year = {2015},
}

@article{nanda2023progress,
  title={Progress measures for grokking via mechanistic interpretability},
  author={Nanda, Neel and Chan, Lawrence and Lieberum, Tom and Smith, Jess and Steinhardt, Jacob},
  journal={arXiv preprint arXiv:2301.05217},
  year={2023}
}

@inproceedings{serrurier2023explainable,
  title={On the explainable properties of 1-{L}ipschitz neural networks: An optimal transport perspective},
  author={Serrurier, Mathieu and Mamalet, Franck and Fel, Thomas and B{\'e}thune, Louis and Boissin, Thibaut},
  booktitle={Neural Information Processing Systems (NeurIPS)},
  year={2023}
}

@inproceedings{srinivas2023models,
  title={Which models have perceptually-aligned gradients? {An} explanation via off-manifold robustness},
  author={Srinivas, Suraj and Bordt, Sebastian and Lakkaraju, Himabindu},
  booktitle={Neural Information Processing Systems (NeurIPS)},
  year={2023}
}

@inproceedings{
tsipras2018robustness,
title={Robustness May Be at Odds with Accuracy},
author={Dimitris Tsipras and Shibani Santurkar and Logan Engstrom and Alexander Turner and Aleksander Madry},
booktitle={International Conference on Learning Representations (ICLR)},
year={2019},
}

@article{kaur2019perceptually,
  title={Are perceptually-aligned gradients a general property of robust classifiers?},
  author={Kaur, Simran and Cohen, Jeremy and Lipton, Zachary C},
  journal={arXiv preprint arXiv:1910.08640},
  year={2019}
}

@article{krishna2024disagreement,
  title={The Disagreement Problem in Explainable Machine Learning: A Practitioner's Perspective},
  author={Krishna, Satyapriya and Han, Tessa and Gu, Alex and Wu, Steven and Jabbari, Shahin and Lakkaraju, Himabindu},
  journal={Transactions on Machine Learning Research ({TMLR})},
  year={2024}
}

@article{uk_gov_review,
title={Review into bias in algorithmic decision-making},
author={{UK Centre for Data Ethics and Innovation}},
url={https://www.gov.uk/government/publications/cdei-publishes-review-into-bias-in-algorithmic-decision-making},
year={2020}
}

@article{hamon2022bridging,
  title={Bridging the gap between AI and explainability in the GDPR: towards trustworthiness-by-design in automated decision-making},
  author={Hamon, Ronan and Junklewitz, Henrik and Sanchez, Ignacio and Malgieri, Gianclaudio and De Hert, Paul},
  journal={IEEE Computational Intelligence Magazine},
  volume={17},
  number={1},
  pages={72--85},
  year={2022}
}

@inproceedings{selbst2018meaningful,
  title={“Meaningful information” and the right to explanation},
  author={Selbst, Andrew and Powles, Julia},
  booktitle={Conference on Fairness, Accountability and Transparency (FAccT)},
  year={2018},
}

@inproceedings{hacker2020varieties,
  title={Varieties of AI explanations under the law. From the GDPR to the AIA, and beyond},
  author={Hacker, Philipp and Passoth, Jan-Hendrik},
  booktitle={International workshop on extending explainable AI beyond deep models and classifiers},
  pages={343--373},
  year={2020},
  organization={Springer}
}

@article{ebers2020regulating,
  title={Regulating explainable AI in the European Union. An overview of the current legal framework (s)},
  author={Ebers, Martin},
  journal={Liane Colonna/Stanley Greenstein (eds.), Nordic Yearbook of Law and Informatics},
  year={2020}
}

\appendix
\newpage
\centerline{\LARGE Appendix:   mathematical details and proofs}

\section{Design choices for the key definitions}

\label{sec-definition-choices}

We have seen that the notion of informative explanations allows us to prove a surprisingly large number of results about vastly different types of post-hoc explanations on tabular data. We now take a step back and discuss some of the design choices we made in the definitions. 

\paragraph{Complexity of $f$.} It is worthwhile to point out that our approach does not try to define the complexity of an individual function $f$ (say, in terms of its variation, or by compression arguments) and then derive guarantees about its explanations. 
This is in line with standard learning theory, which found that to argue about learning an unknown function is much easier if we study function classes rather than individual functions (see 
\citealp{shalev2014understanding} for an introduction to learning theory). \\

\paragraph{Prior knowledge.} Whether an explanation is informative or not depends on our prior knowledge $\Fcal$. Intuitively, it makes sense that if we want to define whether an explanation tells us something new, we need to relate to the knowledge that we have before receiving the explanation. We have seen that this can make all the difference: depending on $\Fcal$, the same explanation can be trivial, non-trivial, uninformative, or informative. 

\paragraph{Non-trivial vs.\ informative.} The distinction between non-trivial and informative explanations is crucial and subtle at the same time. Almost every explanation excludes at least a single function, causing $\Fe$ to be a proper subset of $\Fp$. However, for an uncountable function space, this might be meaningless information. One needs a stronger way of comparing the two classes to capture whether the explanation actually tells us something substantial about the decision function. This leads to our definition of informativeness. As a concrete example of a non-trivial yet uninformative explanation, we refer to the gradients on piecewise constant functions in Section \ref{subsection-first-examples}.

\paragraph{Comparing the classes $\Fp$ and $\Fe$ through prediction behavior.} 
As stated, we need a more sophisticated way to compare the two classes $\Fp$ and $\Fe$ than strict inclusion. Of course, there exist plenty of options for that. We chose one that measures whether the induced space $\Fe$ has a more limited prediction behavior. We argue that in many situations, practitioners are interested in better understanding how the model arrived at its prediction, which means understanding how the model behaves for other data points, especially those very similar to $x_0$. See Section \ref{sec-connection-to-real-world} for a more detailed discussion. For an uninformative explanation, the model could still emulate the same prediction behavior on almost surely any finite sample after seeing the explanation as before.

\paragraph{Nature of information.} Our notion of informative explanations is a purely mathematical, non-constructive concept. It does not tell us what an explanation's information consists of, or how this information could be extracted from the explanation (neither by an algorithm nor by a human user). We only know that the prediction behavior of our function is \emph{somehow} restricted in \emph{some} way. However, this is a feature and not a bug: By using this framework we do not have to specify what kind of information we might be looking for, and the general setup makes our negative results more powerful.

\paragraph{Relation to faithfulness.} The notion of informative explanations encompasses and is more general than many of the suggested formulations of faithfulness. Intuitively, faithfulness tries to assert that the explanation is related to the underlying function $f$ (see, for example, \citet{bhatt2021evaluating, Dasgupta22, velmurugan2021evaluatingfidelity}). It is clear from our definitions that if this is not the case, then an explanation cannot be informative (see example of $f$-independent explanations in Section~\ref{subsection-first-examples}). However, there exist explanations that are perfectly faithful, yet obviously do not provide any tangible information (see example of the tautological explanation  in Section~\ref{subsection-first-examples}). That means, our notion of informativeness is stronger than the classical idea of faithfulness, as the explanation must not only be grounded in the function but also contain information not present in the prediction.

\section{Lemma about Rademacher complexities}\label{sec:lemma-rademacher}

Let $n\in\mathbb{N}$ and let $A$ be an event with positive probability with respect to the distribution $(P\otimes\text{Rad})^{\otimes n}$, where $\text{Rad}$ denotes the Rademacher distribution. Let $\Fcal$ be a function class and define the conditional Rademacher complexity as
\[
R_n(\Fcal|A):=\E_{x,\sigma}\left(\sup_{f\in\Fcal}\frac{1}{n}\sum_{i=1}^n \sigma_i f(x_i)\Bigg|A\right).
\]

\begin{lemma}[\textbf{From conditional to unconditional Rademacher complexity}]
\label{lemma-rademacher}
Let $\Gcal\subseteq\Fcal$ be function classes, let $R_n(\Fcal)$ be finite, and let $A$ be an event with positive probability. If $R_n(\Gcal|A)<R_n(\Fcal|A)$, then $R_n(\Gcal)<R_n(\Fcal)$.
\end{lemma}

\begin{proof}
We can decompose $R_n(\Gcal)$ in the following way:
\begin{align*}
	R_n(\Gcal)&=\E_x\left(\E_\sigma\left(\sup_{g\in\Gcal}\frac{1}{n}\sum_{i=1}^n \sigma_i g(x_i)\right)\right)  \\
	&=\E_{x,\sigma}\left(\sup_{g\in\Gcal}\frac{1}{n}\sum_{i=1}^n \sigma_i g(x_i)\right) \\
	&=\Pr(A)\cdot R_n(\Gcal|A)+\Pr(A^c)\cdot R_n(\Gcal|A^c).
\end{align*}
A similar decomposition holds for $\Fcal$:
\begin{align*}
	R_n(\Fcal)=\Pr(A)\cdot R_n(\Fcal|A)+\Pr(A^c)\cdot R_n(\Fcal|A^c).
\end{align*}
Taking the difference of the two yields
\begin{align*}
	R_n(\Fcal)-R_n(\Gcal)=&\underbrace{\Pr(A)}_{>0\text{ by assumption}}\cdot\underbrace{(R_n(\Fcal|A)-R_n(\Gcal|A))}_{>0\text{ by assumption}}\\&+\Pr(A^c)\cdot\underbrace{(R_n(\Fcal|A^c)-R_n(\Gcal|A^c))}_{\geq0\text{ since }\Gcal\subseteq\Fcal}>0.
\end{align*}

\end{proof}

\section{Proofs about gradient explanations}

\subsection{Proof that gradient explanations for trees are not informative, Section \ref{subsection-first-examples}}
\label{appendix-gradient-for-trees}
\begin{proof}
For the proof we will use the convention that a tree uses the split convention ``$< t$'' and ``$\ge t$'', where $t$ is the splitting threshold. For the convention ``$\le t$'' and ``$>t$'', the proof works analogously. Consider $n\in\N$ and points $x_1,...,x_n$. If a tree $g\in\Fp^{x_0}$ has no cuts going through $x_0$, we have $g\in\Fe^{x_0}$. Otherwise, we will construct a function in $\Fe^{x_0}$ that has the same predictions as $g$ on $x_1,...,x_n$. 
For every cut that goes through $x_0$, say using feature $j$, we consider the $j$-th feature within all the points closest to the cut from below, that is
$$c_j:=\max_{i=1,...,n}\left\{x_i^{(j)}\condon x_i^{(j)}\le x_0^{(j)}\right\}.$$
If there is no point $x_i$ with $x_i^{(j)}\le x_0^{(j)}$, we simply set $c_j=0$ as we are dealing with $\Xcal=[0,1]^d$. Note that $c_j<x_0^{(j)}$ with probability 1 over the choice of the sample. Then, we replace the threshold $x_0^{(j)}$ of the cut with some other value in $(c_j, x_0^{(j)})$.  This ensures that $x_0$ as well as all $x_1,...,x_n$ still flow down the same branch as before, meaning the predictions of $x_0,...,x_n$ are unchanged. After doing so for every critical cut, we have constructed a new tree which lies in $\Fe$ and has the same predictions for all $x_1,...,x_n$ as $g$. This completes the proof. 
\end{proof}

\subsection{Proof of Proposition \ref{prop-gradient-linear-and-noisy} (Linear and noisy linear functions)}\label{proof-prop-gradient-linear}

\begin{proof}\
\begin{enumerate}
    \item We will prove Rademacher-informativeness as defined in Section \ref{sec:alternative definition rademacher}. As shown there, this implies informativeness. Let $\Fcal$ be the class of linear functions $\Lcal^M$ (Figure~\ref{fig-linear-vs-noisy-linear}, left panel), and let $f\in\Fcal$. Let $x_0\in\mathrm{int}(\Xcal)$, let $n\in\mathbb{N}$, and let $v:=\nabla f(x_0)$. Note that we may ignore the case of $x_0$ sitting on the boundary of $\Xcal$ as it has probability zero. Since $\Xcal$ is compact and $\|w\|,|b|<M$, the function class $\Fcal$ is uniformly bounded, therefore, $R_n(\Fcal)<\infty$. Hence, we can use Lemma \ref{lemma-rademacher} to show informativeness. First consider the case $v\ne 0$ and let $A$ be the event that $v^T(x_i-x_0)>0$ and $\sigma_i=+1$ for all $i\in[n]$. This event has positive probability. We will now show that $R_n(\Fe^{x_0}|A)<R_n(\Fp^{x_0}|A)$.
    As we have seen before, $\Fe^{x_0}$ consists of the single function $f(x)=w^Tx+b$. Let $\delta$ be a constant and define $\varphi_\delta(x):=f(x)+ \delta v^T(x-x_0)$. Notice that $\varphi_\delta(x_0)=f(x_0)$ holds for any $\delta$. Because of the condition $\|w\|,|b|<M$, we can always choose $\delta^*>0$ so that $\varphi_{\delta^*}\in\Fp^{x_0}$. If we are in the setting of event $A$, then for any sample $x_1,\dots,x_n$ it holds that
    \[
    \frac{1}{n}\sum_{i=1}^n\sigma_i\varphi_{\delta^*}(x_i)>\frac{1}{n}\sum_{i=1}^n\sigma_i f(x_i),
    \]
    which implies $R_n(\Fe^{x_0}|A)<R_n(\Fp^{x_0}|A)$.\\
    If $v=0$, the event $A$ as defined above has probability zero. In this case, we pick a unit vector $u\in\R^d$ and define $A$ as the event that $u^T(x_i-x_0)>0$ and $\sigma_i=1$ for all $i\in [n]$. Defining $\varphi_\delta(x):=f(x)+\delta u^T(x-x_0)$ and $\delta^*>0$ so that $\varphi_{\delta^*}\in\Fp^{x_0}$, we can analogously conclude
    $$\frac1n\sum_i\sigma_i\varphi_{\delta^*}(x_i)=\frac1n\sum_i\sigma_i f(x_i)
    +\frac{\delta^*}{n}\sum_i u^\top(x_i-x_0)>\frac1n\sum_i\sigma_i f(x_i),$$
    implying $R_n(\Fe^{x_0}|A)<R_n(\Fp^{x_0}|A)$.
    \item Let us fix an arbitrary sample $x_1,\dots,x_n\sim \Pr$. For any $g\in\Fp^{x_0}$ we construct a $\tilde g\in\Fe^{x_0}$ such that for all $i=1,...,n$ it holds
        \[
        \tilde{g}(x_i)=g(x_i).
        \]
        To construct $\tilde{g}$, choose some $\delta>0$, such that $x_1,\dots,x_n\notin B_\delta(x_0)$. This is possible with probability $1$ over the choice of $x_1,...,x_n$. Given $g$, we define $\tilde{g}$ as
        \[
        \tilde{g}(x)=\begin{cases}
            f(x) & x\in B_{\delta/2}(x_0) \\
            g(x) & x\notin B_\delta(x_0) \\
            \text{smooth connection} & x\in B_\delta(x_0)\setminus B_{\delta/2}(x_0)
        \end{cases}
        \]
        Outside the $\delta$ ball, we have $\tilde{g}=g$, thus $\forall i\in[n]: g(x_i)=\tilde{g}(x_i)$. As $\tilde{g}=f$ inside the $\delta/2$-radius ball, we have $\tilde{g}(x_0)=f(x_0)$ and $\exgrad(\tilde{g},x_0)=\exgrad(f,x_0)$. Hence, the only thing missing to show that $\tilde{g}\in\Fe^{x_0}$ is that $\tilde{g}$ is actually $\eps$-close to a linear function. Let $h\in\Lcal^M$ be a linear function that $g$ is $\eps$-close to. Outside of $B_\delta(x_0)$, it holds $g=\tilde{g}$ and thus, $|\tilde{g}(x)-h(x)|<\eps$ for all $x\notin B_\delta(x_0)$. Inside of $B_{\delta/2}(x_0)$ we compute
        \begin{align*}
            |\tilde{g}(x)-h(x)| &=
            |f(x)-h(x)|\\
            &=|f(x) - f(x_0)+f(x_0)-g(x)+g(x)-h(x)|\\
            &\le |f(x)-f(x_0)| + |f(x_0)-g(x)| + |g(x)-h(x)|\\
            &= |f(x)-f(x_0)| + |g(x_0)-g(x)| + \underbrace{|g(x)-h(x)|}_{=:\tilde{\eps} < \eps}.
        \end{align*}
        Now, if we choose $\delta$ sufficiently small, we know that because of continuity, the first two terms become smaller than $\eps - \tilde{\eps}$, yielding $|\tilde{g}(x)-h(x)|<\eps$ for all $x\in B_{\delta/2}(x_0)$. As this property holds inside of $B_{\delta/2}(x_0)$ and outside of $B_\delta(x_0)$, we can choose the smooth connection such that it also holds that $|\tilde{g}(x)-h(x)|<\eps$ for all $x \in B_\delta(x_0)\setminus B_{\delta/2}(x_0)$. This finishes the proof.

\end{enumerate}
\end{proof}

\begin{figure}[h]
    \centering
    \includegraphics[width=0.9\linewidth]{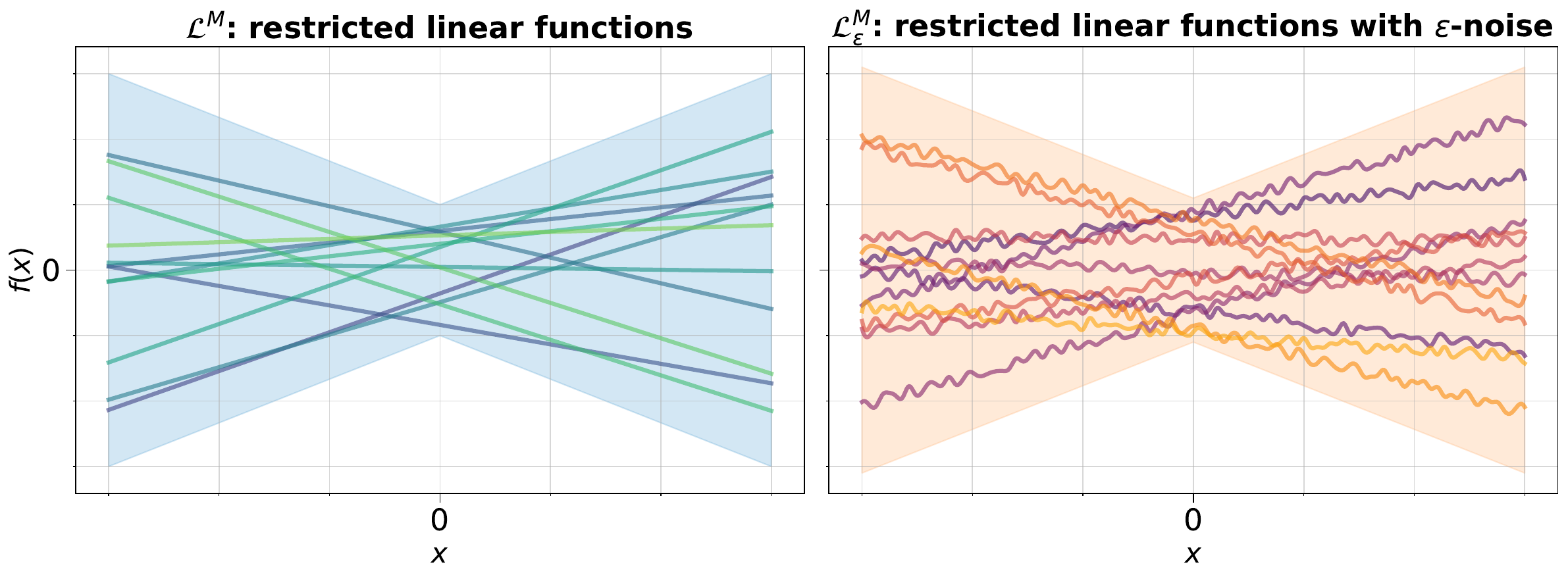}
    \caption{Illustrations of the function spaces $\Lcal^M$ (left) and $\Lcal^M_\varepsilon$ (right) in $1$ dimension. The shaded regions show the image spaces, and the colorful functions represent members of the function classes.}
    \label{fig-linear-vs-noisy-linear}
\end{figure}

\subsection{Proof of Theorem \ref{thm-gradients}}
\label{app:proof-thm-gradients}
\begin{proof}\
    \begin{enumerate}
        \item \begin{description}
            \item[$\Fcal=\Dcal$:] Let $n\in\mathbb{N}$,  $f\in \Fcal$ and $x_0\in\mathrm{int}(\Xcal)$. Note that we can ignore the case of $x_0$ sitting on the boundary of $\Xcal$ as it has probability zero. Let us fix an arbitrary sample $x_1,\dots,x_n\sim \Pr$.
            For any $g\in\Fp^{x_0}$ we construct a $\tilde g\in\Fe^{x_0}$ such that $\tilde{g}(x_i)=g(x_i)$.
            To construct $\tilde{g}$, choose some $\varepsilon>0$, such that $x_1,\dots,x_n\notin B_\varepsilon(x_0)$. This is possible with probability $1$ over the choice of $x_1,...,x_n$. Given $g$, we define $\tilde{g}$ as
            \[
            \tilde{g}(x)=\begin{cases}
            	f(x) & x\in B_{\varepsilon/2}(x_0) \\
            	g(x) & x\notin B_\varepsilon(x_0) \\
            	\text{smooth connection} & x\in B_\varepsilon(x_0)\setminus B_{\varepsilon/2}(x_0)
            \end{cases}
            \]
            As $\tilde{g}=f$ inside the $\varepsilon/2$-radius ball, we have $\exgrad(\tilde{g},x_0)=\exgrad(f,x_0)$. Outside the $\varepsilon$ ball, we have $\tilde{g}=g$, thus $\forall i\in[n]: g(x_i)=\tilde{g}(x_i)$. The smooth connection between the two regions ensures that $\tilde{g}$ is differentiable. Therefore, by construction, $\tilde{g}\in\Fe^{x_0}$.

            \item[$\Fcal=\Dcal_\alpha$:] Even if the gradient is continuous and bounded, the situation does not change. The gradient can still \emph{change} arbitrarily fast because the second derivative is not restricted. This means, a function in $\Fe^{x_0}$ can emulate any behavior of a function in $\Fp^{x_0}$ away from $x_0$ while having the desired gradient in $x_0$.\\
            Let $f\in\Fcal$, $x_0\in\mathrm{int}(\Xcal)$ and $n\in\N$. Note that we can ignore the case of $x_0$ sitting on the boundary of $\Xcal$ as it has probability zero.  Let further $x_1,...,x_n\sim \Pr$ be a sample. For an arbitrary $g\in\Fp^{x_0}$ we are going to construct some $\tilde{g}\in\Fe^{x_0}$ that satisfies $\tilde{g}(x_i)=g(x_i)$ for $i=1,...,n$.  \\
            We call $r:=\min_{i=1,...,n}\|x_i-x_0\|$. Note that $r>0$ almost surely over the choice of the $x_i$, so we will assume this in the following. To change $g$ in a very small neighborhood of $x_0$ such that the gradient changes accordingly, take some continuously differentiable $\phi:\R^d\to [0,1]$ with $\phi(x_0)=1$ and $\phi(x)=0$ for all $x\notin B_r(x_0)$. Let further $k\in\N$. Define 
            $$\tilde{g}(x) := g(x) + \frac{1}{k}\sin(k\langle \nabla f(x_0)-\nabla g(x_0), x-x_0\rangle)\phi(x).$$
            We are going to prove that $\tilde{g}$ has the desired properties if we choose $\phi$ and $k$ suitably. As $\phi$ vanishes on all $x_i$, we have $\tilde{g}(x_i)=g(x_i)$. It remains to show that $\tilde{g}\in\Fe^{x_0}$. Clearly, it holds $\tilde{g}(x_0)=g(x_0)$. For the gradient, we compute
            \begin{align*}
                \nabla \tilde{g}(x) = \nabla g(x) &+ \underbrace{\cos(k\langle \nabla f(x_0)-\nabla g(x_0), x-x_0\rangle)\cdot\phi(x)}_{=:\tilde{\phi}(x)} \cdot (\nabla f(x_0)-\nabla g(x_0))\\
                &+ \frac{1}{k}\sin(k\langle \nabla f(x_0)-\nabla g(x_0), x-x_0\rangle)\cdot \nabla\phi(x).
            \end{align*}       
            Plugging in $x=x_0$ leads to $\nabla \tilde{g}(x_0)=\nabla g(x_0)$. It is only left to show $\|\nabla \tilde{g}(x)\|<\alpha$ in order to conclude $\tilde{g}\in\Fe^{x_0}$. We compute
            \begin{align*}
                \|\nabla\tilde{g}(x)\| 
                &= \|\nabla g(x) 
                + \tilde{\phi}(x) \cdot (\nabla f(x_0)-\nabla g(x_0) + \nabla g(x) - \nabla g(x))\\
                &\qquad+ \frac{1}{k}\sin(k\langle \nabla f(x_0)-\nabla g(x_0), x-x_0\rangle)\cdot \nabla\phi(x)\|\\
                &\le \left(1-|\tilde{\phi}(x)|\right)\cdot \|\nabla g(x)\| + |\tilde{\phi}(x)|\cdot \|\nabla f(x_0)\|  + |\tilde{\phi}(x)|\cdot \|\nabla g(x)-\nabla g(x_0)\|\\
                &\qquad +\frac{1}{k}\| \nabla\phi(x)\|\\
                &\le \max\{\|\nabla g(x)\|, \|\nabla f(x_0)\|\}+ |\phi(x)|\cdot \|\nabla g(x)-\nabla g(x_0)\| +\frac{1}{k}\| \nabla\phi(x)\|.
            \end{align*}
            As all our gradients are continuous, we know they attain their maximum norm on $\Xcal$, so $\max\{\|\nabla g(x)\|, \|\nabla f(x_0)\|\} \le C<\alpha$. If we choose $\phi$ such that $\phi(x)<(\alpha-C)/(2\|\nabla g(x)-\nabla g(x_0)\|)$ for all $x\ne x_0$, then the second summand becomes smaller than $(\alpha-C)/2$. We then pick $k$ greater than $2\cdot\max_{x\in\Xcal} \|\nabla\phi(x)\|/(\alpha-C)$, the third summand also becomes smaller than $(\alpha-C)/2$. This proves $\|\nabla \tilde{g}(x)\|< \alpha$ and hence completes the proof.
            
            \item[$\Fcal=\Pcal$:] Let $f\in\Fcal$ and $x_0\in\mathrm{int}(\Xcal)$. Note that we can ignore the case of $x_0$ sitting on the boundary of $\Xcal$ as it has probability zero. To show that $\exgrad(f,x_0)$ is uninformative for every $n\in\N$, it suffices to show that there is some $\tilde{g}\in\Fe^{x_0}$ with $\tilde{g}(x_i)=g(x_i)$. This is clearly the case, since for any $n\in\mathbb{Z}^+$, among polynomials of arbitrary degree, there exists a polynomial that fits the pairs of points and labels $(x_0,f(x_0)),(x_1,g(x_1)),\dots,(x_n,g(x_n))$, as well as the gradient $\nabla f(x_0)$ at $x_0$.
        \end{description}
        \item \begin{description}
            \item[$\Fcal=\Dcal_{\alpha,\beta}$:] We will prove Rademacher-informativeness as defined in Section \ref{sec:alternative definition rademacher}. As shown there, this implies informativeness. Let $f\in\Fcal$, $x_0\in\mathrm{int}(\Xcal)$ and $n\in\N$. Note that we can ignore the case of $x_0$ sitting on the boundary of $\Xcal$ as it has probability zero.  First, we show that $R_n(\Fp^{x_0})<\infty$. Let $h\in\Fp^{x_0}$ and let $\gamma:\mathbb{R}\to\mathbb{R}^d$ be a smooth curve. Then $h\circ\gamma:\mathbb{R}\to\mathbb{R}$ and by the fundamental theorem of calculus
            \[
            h(\gamma(1))-h(\gamma(0))=\int_0^1\frac{\partial}{\partial t}h(\gamma(t))\,\dd t=\int_0^1\nabla h(\gamma(t))^T\dot{\gamma}(t)\,\dd t.
            \]
            For fixed $x,y\in\mathbb{R}^d$ let $\gamma$ be such that $\gamma(t)=ty+(1-t)x$. Then $\gamma(0)=x,\,\gamma(1)=y$ and $\dot{\gamma}(t)=y-x$, therefore
            \[
            h(y)-h(x)=\int_0^1\nabla h(ty+(1-t)x)^T(y-x)\ \dd t.
            \]
            If we subtract $\nabla h(x)^T(y-x)$ from both sides, we get
            \begin{align*}
            	h(y)-h(x)-\nabla h(x)^T(y-x)&=\int_0^1\Big[\nabla h(ty+(1-t)x)^T-\nabla h(x)^T\Big](y-x)\ \dd t \\
            	&=\Big[\int_0^1\nabla h(ty+(1-t)x)-\nabla h(x)\ \dd t\Big]^T(y-x)\\
            	&\leq \Big|\Big|\int_0^1\nabla h(ty+(1-t)x)-\nabla h(x)\ \dd t\Big|\Big|\cdot\|x-y\|\\
            	&\leq \int_0^1\big|\big|\nabla h(ty+(1-t)x)-\nabla h(x)\big|\big|\ \dd t\cdot\|y-x\| \\
            	&\leq\|y-x\|\cdot\int_0^1 \beta\|ty+(1-t)x-x\|\ \dd t \\
            	&=\|y-x\|\cdot\int_0^1 \beta\|t(y-x)\|\ \dd t \\
            	&= \beta\cdot\|y-x\|^2\cdot\int_0^1 t\ \dd t \\
            	&=\frac{\beta}{2}\cdot\|y-x\|^2.
            \end{align*}
            In summary, 
            \begin{equation}\label{smoothness-leq}
            	h(y)\leq h(x)+\nabla h(x)^T(y-x)+\frac{\beta}{2}\|y-x\|^2.
            \end{equation}
            All of the above can be done with $-h$ as well, which yields
            \begin{equation}\label{smoothness-geq}
            	h(y)\geq h(x)+\nabla h(x)^T(y-x)-\frac{\beta}{2}\|y-x\|^2.
            \end{equation}
            Combining Inequalities (\ref{smoothness-leq}) and (\ref{smoothness-geq}), and using the Cauchy-Schwarz inequality we have
            \[
            |h(y)-h(x)|\leq\|\nabla h(x)\|\cdot\|y-x\|+\frac{\beta}{2}\|y-x\|^2.
            \]
            Notice that if we set $x:= x_0$, then due to the compactness of $\Xcal$ and because $\|\nabla h(x)\|\leq\alpha$, the above inequality implies that $\Fp^{x_0}$ is uniformly bounded, thus $R_n(\Fp^{x_0})<\infty$ for any $n\in\mathbb{Z}^+$.
            
            Now we would like to create an event $A$, for which we can show that $R_n(\Fe^{x_0}|A)<R_n(\Fp^{x_0}|A)$, and then apply Lemma \ref{lemma-rademacher} to obtain the result of the proposition. There are two possible cases: $\|\nabla f(x_0)\|>0$ or $\|\nabla f(x_0)\|=0$.
            \\
            \newline
            \textbf{Case 1:} $\|\nabla f(x_0)\|>0$.
            \\
            Let $g\in\Fe^{x_0}$ and let $0<\delta<\|\nabla g(x_0)\|$. Let $$\Ucal := \{u\in\R^d\condon \|u\|=1, \nabla g(x_0)^Tu>\delta\}.$$ Note that $\Ucal\neq\emptyset$, since $\nabla g(x_0)/\|\nabla g(x_0)\|\in\Ucal$. Applying Inequality (\ref{smoothness-geq}) to $g$, we have
            \begin{align*}
            	g(x_0+\lambda u)&\geq g(x_0)+\nabla g(x_0)^T(x_0+\lambda u-x_0)-\frac{\beta}{2}\|\lambda u\|^2\\
            	&=g(x_0)+\lambda\cdot\underbrace{\nabla g(x_0)^Tu}_{>0}-\frac{\beta}{2}\lambda^2,
            \end{align*}
            where $\lambda\geq0$. If we choose $$0<\lambda^*<\frac{2}{\beta}\inf_{u\in\Ucal}\left(\nabla g(x_0)^Tu\right),$$ then $\lambda^*\cdot\nabla g(x_0)^Tu-\frac{\beta}{2}(\lambda^*)^2>0$, and thus moving from $x_0$ in the direction of $u$ by $\lambda$ strictly increases the function value.
            
            Let $n\in\mathbb{Z}^+$, and let $x_1,\dots,x_n$ be a Rademacher sample. Let $A$ be the following event: $\forall i\in[n]:x_i\in\{x_0+\mu u\condon0<\mu\leq\lambda^*,\,u\in\Ucal\}$ and $\forall i\in[n]:\sigma_i=-1$. Event $A$ has a positive probability because the set $\{x_0+\mu u\condon0<\mu\leq\lambda^*,\,u\in\Ucal\}$ defines a cone with positive volume. Let $\tilde g(x):=-g(x)+2g(x_0)$. We observe that $\tilde g\in\Fp^{x_0}$ and that $\tilde g\notin\Fe^{x_0}$, as $\nabla\tilde g(x_0)=-\nabla g(x_0)$. Due to the above argument, we have
            $$\frac{1}{n}\sum_{i=1}^n\sigma_ig(x_i)<\frac{1}{n}\sum_{i=1}^n \sigma_ig(x_0)+\underbrace{\frac{1}{n}\sum_{i=1}^n\sigma_i\left(\mu_i\cdot\nabla g(x_0)^Tu_i-\frac{\beta}{2}\mu_i^2\right)}_{C(x)<0}$$
            and
            $$\frac{1}{n}\sum_{i=1}^n\sigma_ig(x_i)-2C(x)<\frac{1}{n}\sum_{i=1}^n\sigma_i\tilde g(x_i).$$
            Since the constant $C(x)$ does not depend on $g$, only on the sample $x_1,\dots,x_n$, it holds that 
            $$\sup_{g\in\Fe^{x_0}}\frac{1}{n}\sum_{i=1}^n \sigma_ig(x_i)<\sup_{\tilde g\in\Fp^{x_0}}\frac{1}{n}\sum_{i=1}^n\sigma_i\tilde g(x_i),$$
            which implies $R_n(\Fe^{x_0}|A)<R_n(\Fp^{x_0}|A)$.\\
            \newline
            \textbf{Case 2:} $\|\nabla f(x_0)\|=0$.\\
            Let $g\in\Fe^{x_0}$. If $\|\nabla f(x_0)\|=0$, then by Inequality (\ref{smoothness-geq}) we have 
            \begin{align*}
            	g(x_0+\lambda u)&\geq g(x_0)-\frac{\beta}{2}\|\lambda u\|^2\\
            	&=g(x_0)-\frac{\beta}{2}\lambda^2,
            \end{align*}
            for any $u\in\R^d$, $\|u\|=1$ and any $\lambda\geq0$. Let $0\neq v\in\R^d$ with $\|v\|\leq\alpha$, let $0>\delta>-\|v\|,$ and define $$\Ucal := \{u\in\R^d\condon \|u\|=1, v^Tu<\delta\}.$$ Again $\Ucal\neq\emptyset$ because $-v/\|v\|\in\Ucal$. Let $\tilde g\in\Fp^{x_0}$ be such that $\nabla\tilde g(x_0)=v$, thus $\tilde g\notin\Fe^{x_0}$. Applying Inequality (\ref{smoothness-leq}) to $\tilde g$, we obtain
            \begin{align*}
            	\tilde g(x_0+\lambda u)&\leq \tilde g(x_0)+\lambda\cdot \underbrace{v^Tu}_{<0}+\frac{\beta}{2}\lambda^2,
            \end{align*}
            where $\lambda\geq0$. If we choose $$0<\lambda^*<-\frac{1}{\beta}\inf_{u\in\Ucal}\left(v^Tu\right),$$ then $-\frac{\beta}{2}(\lambda^*)^2>\lambda^*\cdot v^Tu+\frac{\beta}{2}(\lambda^*)^2$.
            
            Let $n\in\mathbb{Z}^+$, and let $x_1,\dots,x_n$ be a Rademacher sample. Let $A$ be the following event: $\forall i\in[n]:x_i\in\{x_0+\mu u\condon0<\mu\leq\lambda^*,\,u\in\Ucal\}$ and $\forall i\in[n]:\sigma_i=-1$. Event $A$ has a positive probability because the set $\{x_0+\mu u\condon0<\mu\leq\lambda^*,\,u\in\Ucal\}$ defines a cone with positive volume. Because of the above arguments, we have
            $$\frac{1}{n}\sum_{i=1}^n\sigma_ig(x_i)<\frac{1}{n}\sum_{i=1}^n\sigma_ig(x_0)-\underbrace{\frac{1}{n}\sum_{i=1}^n\sigma_i\frac{\beta}{2}\mu_i^2}_{C(x)}.$$
            It also holds that
            \begin{align*}\frac{1}{n}\sum_{i=1}^n\sigma_i\tilde g(x_i)&\geq\frac{1}{n}\sum_{i=1}^n\sigma_i\left(g(x_0)+\mu_i\cdot v^Tu_i+\frac{\beta}{2}\mu_i^2\right)\\&>\frac{1}{n}\sum_{i=1}^n\sigma_i\left(g(x_0)+\frac{\beta}{2}\mu_i^2\right)\\&=\frac{1}{n}\sum_{i=1}^n\sigma_ig(x_0)+\frac{1}{n}\sum_{i=1}^n\sigma_i\frac{\beta}{2}\mu_i^2.
            \end{align*}
            Therefore,
            $$\frac{1}{n}\sum_{i=1}^n\sigma_ig(x_i)-2C(x)<\frac{1}{n}\sum_{i=1}^n\sigma_i\tilde g(x_i).$$
            Since the constant $C(x)$ does not depend on $g$, only on the sample $x_1,\dots,x_n$, it holds that 
            $$\sup_{g\in\Fe^{x_0}}\frac{1}{n}\sum_{i=1}^n \sigma_ig(x_i)<\sup_{\tilde g\in\Fp^{x_0}}\frac{1}{n}\sum_{i=1}^n\sigma_i\tilde g(x_i),$$
            which implies $R_n(\Fe^{x_0}|A)<R_n(\Fp^{x_0}|A)$. 
            \item[$\Fcal=\Pcal_{D,M}$:] We will prove Rademacher-informativeness as defined in Section \ref{sec:alternative definition rademacher}. As shown there, this implies informativeness. We are going to prove that $\Pcal_{D,M}\subseteq \Dcal_{\alpha,\beta}$. First, we observe that if $\varphi\in\mathcal{P}_{D,M}$, then $\varphi\in C^\infty$ as well. Look at a partial derivative of $\varphi$:
            \[
            \frac{\partial}{\partial x^{(i)}}\varphi(x)=\sum_{|\nu|\leq D,\,\nu_i\geq1} \underbrace{a_\nu}_{\text{bounded by assumption}}\cdot \underbrace{\nu_i}_{\leq D}\cdot \underbrace{\left(x^{(1)}\right)^{\nu_1}\cdots \left(x^{(i)}\right)^{\nu_i-1}\cdots \left(x^{(d)}\right)^{\nu_d}}_{\text{bounded on }\Xcal\text{ (compact)}}.
            \]
            This holds for all $i\in[d]$, therefore, there exists $M'$ so that all partial derivatives of $\varphi$ are in $\mathcal{P}_{D-1,M'}$. The previous equation also implies that there exists $\alpha\geq0$ such that for all $\varphi\in\Pcal_{D,M}$ and for all $x\in\Xcal$ it holds that $\|\nabla\varphi(x)\|\leq\alpha$. 
            
            Using a similar argument, we can conclude that there exists $M''$ such that all second-order partial derivatives of $\varphi$ are in the class $\mathcal{P}_{D-2,M''}$ and are bounded. This means that the elements of the Hessian $H_\varphi$ of any function $\varphi\in\mathcal{P}_{D,M}$ are bounded. In particular, this implies that $\|H_\varphi\|_F$ is bounded, which is an upper bound on the spectral norm of the Hessian. Thus, we get that functions in $\mathcal{P}_{D,M}$ have bounded curvature.
            \\
            As a next step, we prove that bounded curvature implies Lipschitz continuous gradients. Let $\varphi$ be sufficiently differentiable. It is to be shown that
            \[
            \forall x\in\Xcal:\|H_\varphi(x)\|_2\leq\beta \Rightarrow \forall x,y\in\Xcal: \|\nabla\varphi(x)-\nabla\varphi(y)\|\leq\beta\cdot\|x-y\|.
            \]
            Let $x,y\in\Xcal$, and let $\gamma(t):= x+t(y-x),\, t\in[0,1]$. Define $g(t):=\nabla \varphi(\gamma(t))$. Then
            \begin{align*}
            	\|\nabla \varphi(y)-\nabla \varphi(x)\|&=\|g(1)-g(0)\|\\
            	&=\Big|\Big|\int_0^1\dot{g}(t)\, \dd t\Big|\Big|\\
            	&=\Big|\Big|\int_0^1H_\varphi(\gamma(t))\cdot\dot{\gamma}(t)\, \dd t\Big|\Big|\\
            	&=\Big|\Big|\int_0^1H_\varphi(\gamma(t))\cdot(y-x)\,\dd t\Big|\Big|\\
            	&\leq\int_0^1\|H_\varphi(\gamma(t))\cdot(y-x)\|\, \dd t \\
            	&\leq \int_0^1 \|H_\varphi(\gamma(t))\|_2\cdot\|y-x\|\,\dd t\\
            	&\leq\int_0^1\beta\cdot\|y-x\|\,\dd t\\
            	&=\beta\cdot\|y-x\|,
            \end{align*}
            which is what we wanted to show.
            
            In summary, we have that $\Pcal_{D,M}\subseteq\Dcal_{\alpha,\beta}$ for some $\alpha$ and $\beta$. From here onwards, the same proof works as in the $\Fcal=\Dcal_{\alpha,\beta}$ case.
        \end{description}
    \end{enumerate}
\end{proof}

\section{Extended results on gradient explanations}

\subsection{Informativeness over piecewise linear functions on a known grid}\label{section-gradient-piecewise-linear}

\begin{proposition}[\textbf{Gradient explanations on piecewise linear functions on a known grid are informative}]\label{prop-gradient-piecewise-linear}
Consider a compact data space $\Xcal\subseteq\R^d$ with a probability distribution $\Pr$ that has a positive density. Let $\Fcal$ be the class of bounded piecewise linear functions on a known grid $\Xcal_1,\dots,\Xcal_m\subseteq\Xcal$:
\[
\Fcal:=\left\{f:\Xcal\to\R\ \big|\ \forall k\in[m]:f\big|_{\Xcal_k}(x)=w_k^Tx+b_k, \ \|w_k\|, |b_k|<M\right\}.
\]
Over this space, gradient explanations $\exgrad(f,x_0)$ are informative for all $f\in\Fcal$, all $x_0\in\operatorname{int}(\Xcal)$, and all $n\in\mathbb{N}$.
\end{proposition}

\begin{proof}
We will prove Rademacher-informativeness as defined in Section \ref{sec:alternative definition rademacher}. As shown there, this implies informativeness. Since $\Xcal$ is compact and $\|w_k\|,|b_k|<M$ for all $k\in[m]$, the function class $\Fcal$ is uniformly bounded, therefore, $R_n(\Fcal)<\infty$. Hence, we can use Lemma \ref{lemma-rademacher} to show informativeness. 

Let $x_0\in\Xcal_k$ and let $A$ be the event that all Rademacher variables fall into $\operatorname{int}(\Xcal_k)$. This is exactly the setting of Proposition \ref{prop-gradient-linear-and-noisy}, which proves $R_n(\Fe^{x_0}|A)<R_n(\Fp^{x_0}|A)$, and thus informativeness.
\end{proof}

\subsection{Results using only the largest gradient component}\label{appendix-gradient-largest-component}

We believe that many results in Section \ref{sec-gradient-explanations} still hold if only the largest component of the gradient and its index are given as an explanation: $$\Ecal(f,x_0):=\left(\max_{j\in[d]}\left|\frac{\partial}{\partial x^{(j)}}f(x_0)\right|,\operatorname*{argmax}_{j\in[d]}\left|\frac{\partial}{\partial x^{(j)}}f(x_0)\right|\right).$$ In the following, we illustrate how a proposition and its proof might look in this case.\\

\begin{proposition}[\textbf{Largest gradient component explanations on the class of linear functions are informative}] \label{prop-gradient-component-linear}
Consider a compact data space $\Xcal\subseteq\R^d$ with a probability distribution $\Pr$ that has a positive density. Let $\Fcal$ be the class of linear functions with bounded coefficients
\[
\Lcal^M:=\left\{f:\Xcal\to\mathbb{R}\condon f(x)=w^Tx+b, \ \|w\|,|b|<M\right\}.
\]
Over this space,  
largest gradient component explanations $\ex(f,x_0)$ are informative for all $f \in \Fcal$, for all $x_0 \in \operatorname{int}(\Xcal)$, and all $n\in\mathbb{N}$.
\end{proposition}
\begin{proof}
We will prove Rademacher-informativeness as defined in Section \ref{sec:alternative definition rademacher}. As shown there, this implies informativeness. Without loss of generality assume that $\operatorname{argmax}_{j\in[d]}\left|\frac{\partial}{\partial x^{(j)}}f(x_0)\right|=1$. Let $$g(x)=w^Tx+b,\quad g\in\Fe^{x_0}.$$ We denote the first standard basis vector by $e_1:=(1,0,\dots,0)$ and construct $$\tilde g(x)=g(x)+\delta e_1^T(x-x_0),\quad \tilde g\in\Fp^{x_0},$$ where $\delta>0$ is such that $\|\nabla\tilde g(x)\|=\|\nabla g(x)+\delta e_1\|<M$.
Since $\Fcal$ is uniformly bounded, and thus $R_n(\Fcal)<\infty$, we can use Lemma \ref{lemma-rademacher} to prove the result. Let $n\in\mathbb{N}$ and let $x_1,\dots,x_n$ be Rademacher variables. Let $A$ be the event that $\forall i\in[n]: x_i\in\{x\in\operatorname{int}(\Xcal)\condon x^{(1)}>x_0^{(1)}\}$ and $\sigma_i=+1$. Then it holds that $$\frac{1}{n}\sum_{i=1}^n\sigma_ig(x_i)<\frac{1}{n}\sum_{i=1}^n\sigma_i\tilde g(x_i),$$
and we have $R_n(\Fe^{x_0}|A)<R_n(\Fp^{x_0}|A).$
\end{proof}
\subsection{Gradient explanations are informative if they are locally stable}
\label{appendix-gradients-locally-stable}
In Section \ref{sec-gradient-explanations}, we have seen that one way to make explanations informative is to restrict the space $\Fcal$ of functions. In this section, we point out a different approach: Instead of restricting the function space to make the explanation $\ex$ informative, we can also keep the function space, but ``augment'' the explanation $\ex$ with further information. A particularly promising approach consists of asserting that explanations only vary slowly around the data point of interest.  

\begin{definition}[\textbf{Locally stable explanations}] Consider an explanation algorithm that generates unique explanations $\ex(f,x)$. Let $|\!|\!| \cdot |\!|\!| $ be a norm that allows for measuring the distance between two explanations. For $r>0, \delta >0$, we call an explanation $\ex(f,x_0)$ \textbf{locally stable} with respect to $r$ and $\delta$ if the explanation only varies by a factor of $\delta$ in the ball of radius $r$ around $x_0$, that is
\ba
\forall x,y \in B_r(x_0): \; 
|\!|\!| \ex(f,x) - \ex(f,y)|\!|\!|  \leq \delta\cdot\|x-y\|. 
\ea
Based on an explanation algorithm $\ex$, we define the extended explanation algorithm $\ex^\#$: It returns a triple $(\ex(f,x_0), r, \delta)$ such that $\ex(f,x_0)$ is locally stable with respect to $r$ and $\delta$ or, if no such $r$ and $\delta$ exist, the algorithm simply returns $\ex^\#(f,x_0)=(\ex(f,x_0), \mathrm{NaN}, \mathrm{NaN})$.
\end{definition}

This notion captures the intuitively appealing idea that an explanation should not change wildly between neighboring data points, but should be somewhat consistent in local neighborhoods 
\citep{alvarez2018towards}. Note that local stability is different in spirit from what people often define as robust explanations: the robustness of an explanation algorithm measures how much an explanation of a fixed data point $x_0$ changes if we slightly perturb the function,  slightly change the underlying data distribution, or slightly change parameters of our explanation algorithm. Stability, in contrast, measures how much the explanations change if we fix $f$, the distribution, and the algorithm, but slightly vary the input data point. 

As the next proposition shows, locally stable gradient explanations are informative even over the space of all differentiable functions with bounded gradients. 

\begin{proposition}[\textbf{Locally stable gradient explanations are informative}]\label{prop-gradient-locally-similar}
Consider a compact data space $\Xcal\subseteq\R^d$ with a probability distribution $\Pr$ that has a positive density. Let 
$\Fcal=\Dcal_\alpha$ be the class of continuously differentiable functions with bounded gradients. 
Consider a point $x_0 \in \operatorname{int}(\Xcal)$ and $r>0$ such that $B_r(x_0)\subseteq \Xcal$, and let $\delta>0$. Then, 
over the class of $\Dcal_\alpha$, $r$-$\delta$-locally stable gradient explanations $\exgrad^\#(f,x_0)$ are informative for all $f\in\Fcal$, all $n\in\mathbb{N}$.
\end{proposition}

\begin{proof}
We will prove Rademacher-informativeness as defined in Section \ref{sec:alternative definition rademacher}. As shown there, this implies informativeness. Since bounded gradients imply Lipschitz continuity, functions in $\Fcal$ are $\alpha$-Lipschitz continuous. Hence, by the compactness of $\Xcal$ we can conclude that $\Fp^{x_0}$ is uniformly bounded, and thus $R_n(\Fp^{x_0})<\infty$. 
An $r$-$\delta$-locally stable gradient explanation 
$\exgrad^\#(f,x_0)$     not only tells us the gradient $\nabla f(x_0)$ of the function, but also the particular locality parameters $r$ and $\delta$. Hence, for all $\varphi\in\Fe^{x_0}$ we know that 
$$\forall x,y\in B_r(x_0): \|\nabla \varphi(x)-\nabla \varphi(y)\|\leq\delta\|x-y\|.$$ In particular, all functions $\phi \in \Fe^{x_0}$ have $\delta$-Lipschitz continuous gradients.\\
Let $n\in\mathbb{N}$, and let $A$ be the event that all Rademacher variables fall into $B_r(x_0)$. This event has positive probability. Because of Lemma \ref{lemma-rademacher}, in order to prove the statement, it is enough to show that $R_n(\Fe^{x_0}|A)<R_n(\Fp^{x_0}|A)$. But if event $A$ happens, then it is exactly the same as the setting of $\Fcal=\Dcal_{\alpha,\beta}$ in Theorem \ref{thm-gradients}, which proves the result.
\end{proof}

Note that gradient explanations without the locality property are not informative over the class of functions with bounded gradients. Hence, the knowledge about local stability renders explanations informative that would not be informative otherwise.

\section{Proofs about SHAP explanations}

For most of the proofs, we need the so-called efficiency axiom, which states 

\begin{equation}\label{eq-efficiency-axiom}
\Phi_1(x)+...+\Phi_d(x) 
=  f(x)-\E(f(X)),
\end{equation}
see \citet{Lundberg17}.
\subsection{How SHAP explanations restrict the function space}\label{appendix-shap-intuition}
Each interventional SHAP value depends on the marginal expectations $v(J,x_0):=\E_X(f(x_0^{(J)}, X^{(\bar{J})}))$ for all $J\subseteq[d]$. Consequently, knowing the SHAP values of $x_0$ in addition to $f(x_0)$ gives us a linear system of $d$ equations
\begin{align*}
    \Phi_1(x_0) &= \sum_{J\subseteq [d]\setminus\{1\}} {\frac{|J|!(d-|J|-1)!}{d!}}\cdot \Big(v(J\cup \{1\}, x_0)-v(J,x_0)\Big)\\
    &\;\,\vdots\\
    \Phi_d(x_0) &= \sum_{J\subseteq [d]\setminus\{d\}} {\frac{|J|!(d-|J|-1)!}{d!}}\cdot \Big(v(J\cup \{d\}, x_0)-v(J,x_0)\Big)
\end{align*}
and $2^d-1$ marginal expectations as our unknowns (one for each subset of features $J$, where $v([d], x_0)=f(x_0)$ is already known). \\
These restrictions on the marginal expectations of our function do not necessarily make the SHAP explanation informative. Rather, it depends on the function class: In the case of a very rich function class, we might be able to fit a sample of data points and still choose the function such that it has the desired marginal expectations. For weaker function classes, these additional restrictions may limit the ability to fit data points, leading to an informative explanation. \\
The simplest restriction to deal with is the one on $v(\emptyset)=\E(f(X))$ as we concretely know $\Phi_1(x_0)+...+\Phi_d(x_0) =  f(x_0)-\E(f(X))$, see Equation (\ref{eq-efficiency-axiom}) above. Most of our proofs exploit this equation.

\subsection{Proof of Proposition \ref{prop-shap-grid}}
\label{app: proof-shap-grid}

\begin{proof}\
\begin{enumerate}
    \item We will prove Rademacher-informativeness as defined in Section \ref{sec:alternative definition rademacher}. As shown there, this implies informativeness. Let $\Fcal$ be the class of grid functions $\Hcal_k$ with values in $[-1,1]$. Let $f\in\Fcal$, let $x_0\in\Xcal$, and let $n\in\N$. Consider a Rademacher sample $x_1,...,x_n$ such that there lies one point in every grid cell besides the cell of $x_0$. This happens with positive probability if $n\ge k^d-1$. Furthermore, consider labels $\sigma_1,...,\sigma_n$ such that all points within the same cell have the same label. We call this event $A$. Then, there is a function $g\in\Fp^{x_0}$ which can interpolate these labels. However, $g$ is determined on all grid cells, so its expected value is also determined by the $\sigma_i$. %
    The SHAP values of $x_0$ give us the expected value of $f$ as $\Phi_1(x_0)+...+\Phi_d(x_0) =  f(x_0)-\E(f(X))$. Hence, all functions in $\Fe^{x_0}$ have the same expected value as $f$. This means, if we choose the $\sigma_i$ such that $\E(g(X))\ne \E(f(X))$, we may conclude $g\notin\Fe^{x_0}$. In other words, there is a function in $\Fp^{x_0}$, which can interpolate our Rademacher points, but no such function in $\Fe^{x_0}$. So, for these labels, we have
    $$1= \sup_{g\in\Fp^{x_0}} \frac{1}{n}\sum_{i=1}^{n} \sigma_ig(x_i) > \sup_{g\in\Fe^{x_0}} \frac{1}{n}\sum_{i=1}^{n} \sigma_ig(x_i), $$
    and taking expectation over $A$ yields $R_n(\Fp^{x_0}|A)> R_n(\Fe^{x_0}|A)$. As $A$ occurs with positive probability, Lemma~\ref{lemma-rademacher} concludes the proof.
    \item Let $\Fcal$ be the class of mean-centered grid functions $\bar \Hcal_k$ with values in $[-1,1]$. Let $f\in\Fcal$, and let $n\in\N$. We call the grid cells $C_{ij}$, where $C_{ij}:= [(i-1)/k, i/k)\times [(j-1)/k, j/k)$. Let $x_0\in\Xcal$ be the point to be explained on the data manifold, and call its cell $C_{ll}$. 
    The SHAP values of $x_0$ are given by 
    \begin{align*}
    	\Phi_1(x_0,f)&=\frac{1}{2}\left(f(x_0)-\E_{X^{(1)}}\left(f\left(X^{(1)},x_0^{(2)}\right)\right) + \E_{X^{(2)}}\left(f\left(x_0^{(1)},X^{(2)}\right)\right)-\underbrace{\E(f(X))}_{=0}\right)\\
    	\Phi_2(x_0,f)&=\frac{1}{2}\left(f(x_0)-\E_{X^{(2)}}\left(f\left(x_0^{(1)}, X^{(2)}\right)\right) + \E_{X^{(1)}}\left(f\left(X^{(1)}, x_0^{(2)}\right)\right)-\underbrace{\E(f(X))}_{=0}\right).
    \end{align*}
    Note that the SHAP values only depend on the cells $C_{il}, C_{lj}$, for $i,j=1,...,k$, because one of the coordinates plugged into $f$ always comes from $x_0$. For $i,j\neq l$, these cells have probability zero. We also know $f(C_{ll})=f(x_0)$ across $\Fp^{x_0}$. \\
    Take a function $h\in\Fe^{x_0}$. Now, for every $g\in\Fp^{x_0}$ we may define $\tilde{g}$ such that $\tilde{g}=g$ on $C_{ii}$ for $i=1,...,k$ and $\tilde{g} =h$ on $C_{il}, C_{lj}$, for $i,j=1,...,k$. Then, for every $g$, the function $\tilde{g}$ has the same SHAP values as $h$ as they only depend on $C_{il}$ and $C_{lj}$, where $l=1,...,k$. Hence, for every $g\in\Fp^{x_0}$, the according function $\tilde{g}$ lies in $\Fe^{x_0}$. At the same time, our function $\tilde{g}$ has the same predictions as $g$ on almost surely every sample as all probability mass is sitting on $C_{ii}$, where $i=1,...,k$. This implies that SHAP is not informative.
    \end{enumerate}
\end{proof}

\subsection{Proof of Theorem \ref{thm-shap-trees}}
\label{app:proof-shap}
\begin{proof}\
    \begin{enumerate}
        \item \begin{figure}
        	\centering
           \includegraphics[width=0.7\linewidth]{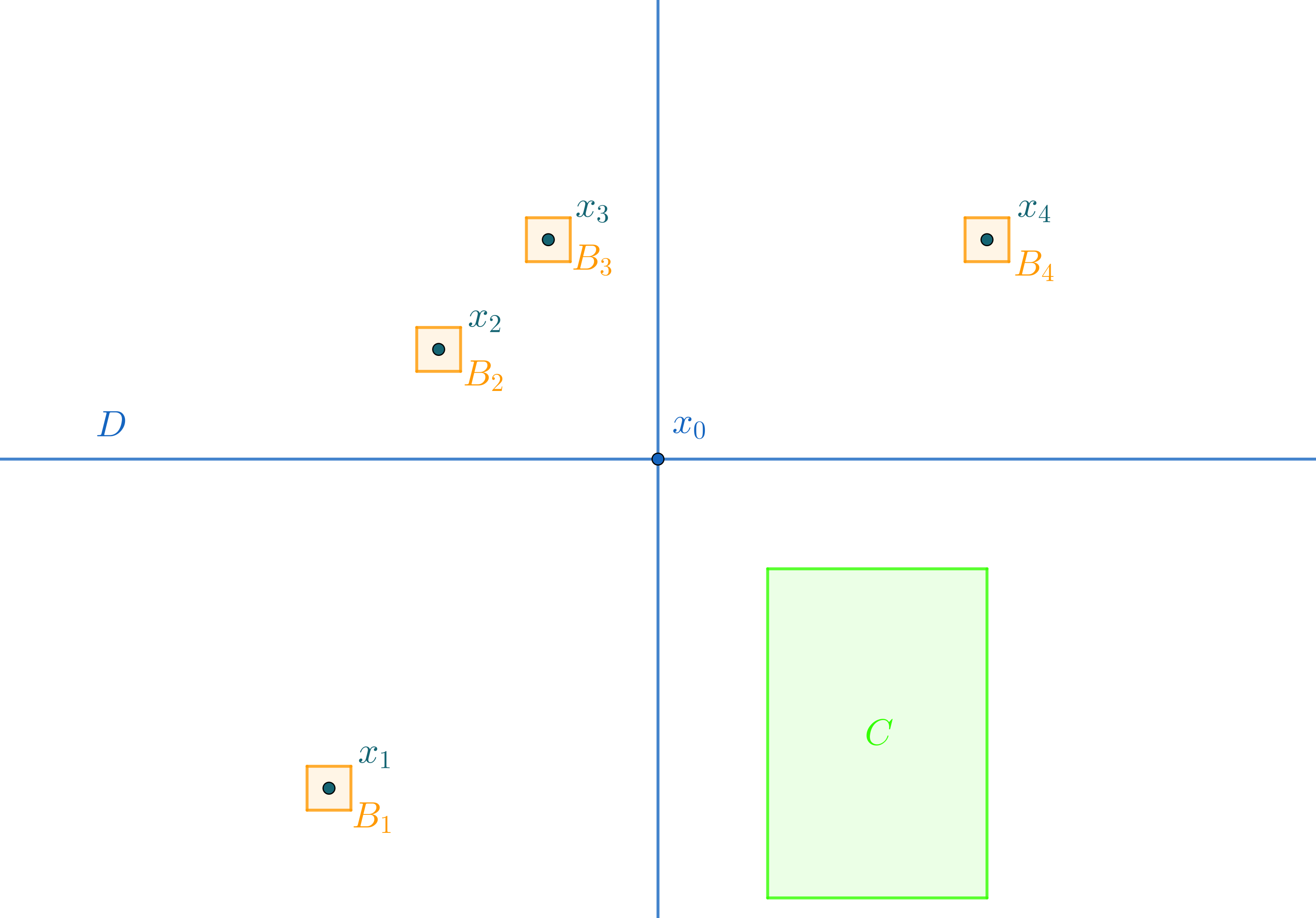}
        	\caption{\textcolor{blue}{$D$} are all points that share coordinates with $x_0$, those are relevant for the SHAP values. The leaves \textcolor{orange}{$B_1,...,B_4$} around $x_1,...,x_4$ are chosen small enough so they do not intersect $D$. On \textcolor{green}{$C$}, we correct for the expected value. }
        	\label{figure-shap-unbounded-trees}
            \end{figure}
            Let $\Fcal$ be the class of trees of arbitrary depth $\Tcal_\infty$, let $f\in\Fcal$, and let $x_0\in\Xcal$. The SHAP values at $x_0$ only depend on the marginal expectations $\E_{X}(f(x^{(J)},X^{(\bar{J})}))$ for all $J\subseteq [d]$. Here, $x^{(J)}$ refers to all the features of $x$ that are in $J$ and $X^{(\bar{J})}$ to these that are not. That means, the SHAP values of $x_0$ only depend on $\E(f(X))$ and the values $f(x)$ for points $x$ that share at least one coordinate with $x_0$. We denote the set of all points that share a coordinate with $x_0$ by $D\subseteq \Xcal$ (\textcolor{blue}{blue} part in Figure \ref{figure-shap-unbounded-trees}).\\
            For $n\in\N$ consider $n$ points $x_1,...,x_n$ and some $g\in\Fp^{x_0}$. Note that with probability 1 over the choice of the sample, the points $x_1,...,x_n$ do not share any coordinate with $x_0$, meaning they do not lie in $D$. \\
            We are going to define some $\tilde{g}$ such that it satisfies $\tilde{g}(x_i)=g(x_i)$ for $i=0,...,n$ while also $\E(\tilde{g}(X))=\E(f(X))$ and $\tilde{g}=f$ on $D$. Note that this completes the proof as the last two properties ensure that the SHAP values of $\tilde{g}$ in $x_0$ are the same as the ones of $f$, so $\tilde{g}\in\Fe^{x_0}$, and also $\tilde{g}$ fits our points. For $x_1,...,x_n$, we define boxes $B_1,...,B_n$ around them (\textcolor{orange}{orange} squares in Figure \ref{figure-shap-unbounded-trees}), small enough such that they do not intersect $D$. We call $p_i:=\Pr(B_i)$. We now define $\tilde{g}(B_i)=g(x_i)$. This is possible as we deal with trees of arbitrary depth. Now, we only have to make sure that $\E(\tilde{g}(X))=\E(f(X))$, which can be achieved by defining $\tilde{g}$ accordingly away from $D$ and the $B_i$. Concretely, we take a rectangle $C$ that does not intersect $D$ or any $B_i$ (\textcolor{green}{green}  rectangle in Figure \ref{figure-shap-unbounded-trees}). We let $\tilde{g}=f$ on $\Xcal\setminus(C\cup \bigcup_{i=1}^nB_i)$. In order to have $\E(\tilde{g}(X))=\E(f(X))$, we must define 
            $$\tilde{g}\equiv \E_{X\sim C}(f(X)) - \frac{\sum_{i=1}^np_i(g(x_i)-\E_{X\sim B_i}(f(X)))}{\Pr(C)}$$
            on $C$. We then compute
            \ba
            \Pr(C)&\cdot \tilde{g}(C) + \sum_{i=1}^n \Pr(B_i)\cdot \tilde{g}(B_i)\\
            &= \Pr(C)\cdot \tilde{g}(C) + \sum_{i=1}^n p_ig(x_i) \\
            &= \Pr(C)\cdot \left(\E_{X\sim C}(f(X)) - \frac{\sum_{i=1}^np_i(g(x_i)-\E_{X\sim B_i}(f(X)))}{\Pr(C)}\right) + \sum_{i=1}^n p_ig(x_i) \\
            &= \Pr(C)\cdot \E_{X\sim C}(f(X)) + \sum_{i=1}^n p_i \E_{X\sim B_i}(f(X)).\\
            \ea
            Together with $\tilde{g}=f$ on $\Xcal\setminus(C\cup \bigcup_{i=1}^nB_i)$, this implies $\E(\tilde{g}(X))=\E(f(X))$.
        \item Let $\Fcal$ be the class $\Gcal_\infty$ of GAMs with trees of arbitrary depth as components. Let $f\in\Fcal$, let $x_0\in\Xcal$, and let $n\in\N$. For some $g\in\Fcal$ we denote its SHAP values at $x_0$ by $\Phi_1^g(x_0),...,\Phi_d^g(x_0)$. Note that due to the GAM structure, they are given by $\Phi^g_j(x_0) = g_j(x_0^{(j)})- \E(g_j(X^{(j)}))$ \citep{BorLux_shap_2023}. \\Consider points $x_1,...,x_n$ and some $g\in\Fp^{x_0}$. We construct a function $\tilde{g}\in\Fe^{x_0}$ with $\tilde{g}(x_i)=g(x_i)$. That means, the SHAP values of $\tilde{g}$ need to coincide with those of $f$.
        To achieve this, set $\tilde{g}_j=f_j$ for $j=2,...,d$, which implies $\Phi^{\tilde{g}}_j(x_0)=\Phi^f_j(x_0)$ for $j=2,...,d$. The component $\tilde{g}_1$ is constructed in the following. Note that, with probability 1 over the choice of the sample, the points $x_0,x_1,...,x_n$ have all different first coordinates. We define small intervals $I_0,I_1,...,I_n$ around $x_0^{(1)},x_1^{(1)},,...,x_n^{(1)},$ with marginal probability $\Pr(X^{(1)}\in I_i)=p_i$ for all $i=0,...,n$. Note that for $p_i$ small enough, these intervals are always disjoint. As our trees are not restricted in depth, we may define $\tilde{g}_1(I_0)=f_1(x_0)$ and  $\tilde{g}_1(I_i)=g(x_i)-\sum_{j=2}^d f_j(x_i)$ for $i=1,...,n$ leading to $\tilde{g}(x_i)=g(x_i)$. On $\R\setminus \bigcup_{i=0}^n I_i$ we define 
        $$\tilde{g}_1\equiv c:=\frac{\E\left(f_1\left(X^{(1)}\right)\right) -  \sum_{i=0}^n p_i\tilde{g}_1(x_i)}{1-\sum_{i=0}^np_i}.$$
        Then, we have
        \ba    \Phi^{\tilde{g}}_1(x_0)&=\tilde{g}_1\left(x_0^{(1)}\right)- \E\left(\tilde{g}_1\left(X^{(1)}\right)\right)\\
        &=f_1\left(x_0^{(1)}\right)- \E\left(\tilde{g}_1\left(X^{(1)}\right)\right)\\
        &=  f_1\left(x_0^{(1)}\right) - \left( \sum_{i=0}^n p_i \tilde{g}_1(x_i) + \left(1-\sum_{i=0}^np_i\right)\cdot c\right)\\
        &= f_1\left(x_0^{(1)}\right) - \E\left(f_1\left(X^{(1)}\right)\right)\\
        &=\Phi^f_1(x_0),
        \ea
        so $\tilde{g}\in\Fe^{x_0}$, completing the proof.
        \item \begin{figure}
        	\centering
        	\includegraphics[width=0.7\linewidth]{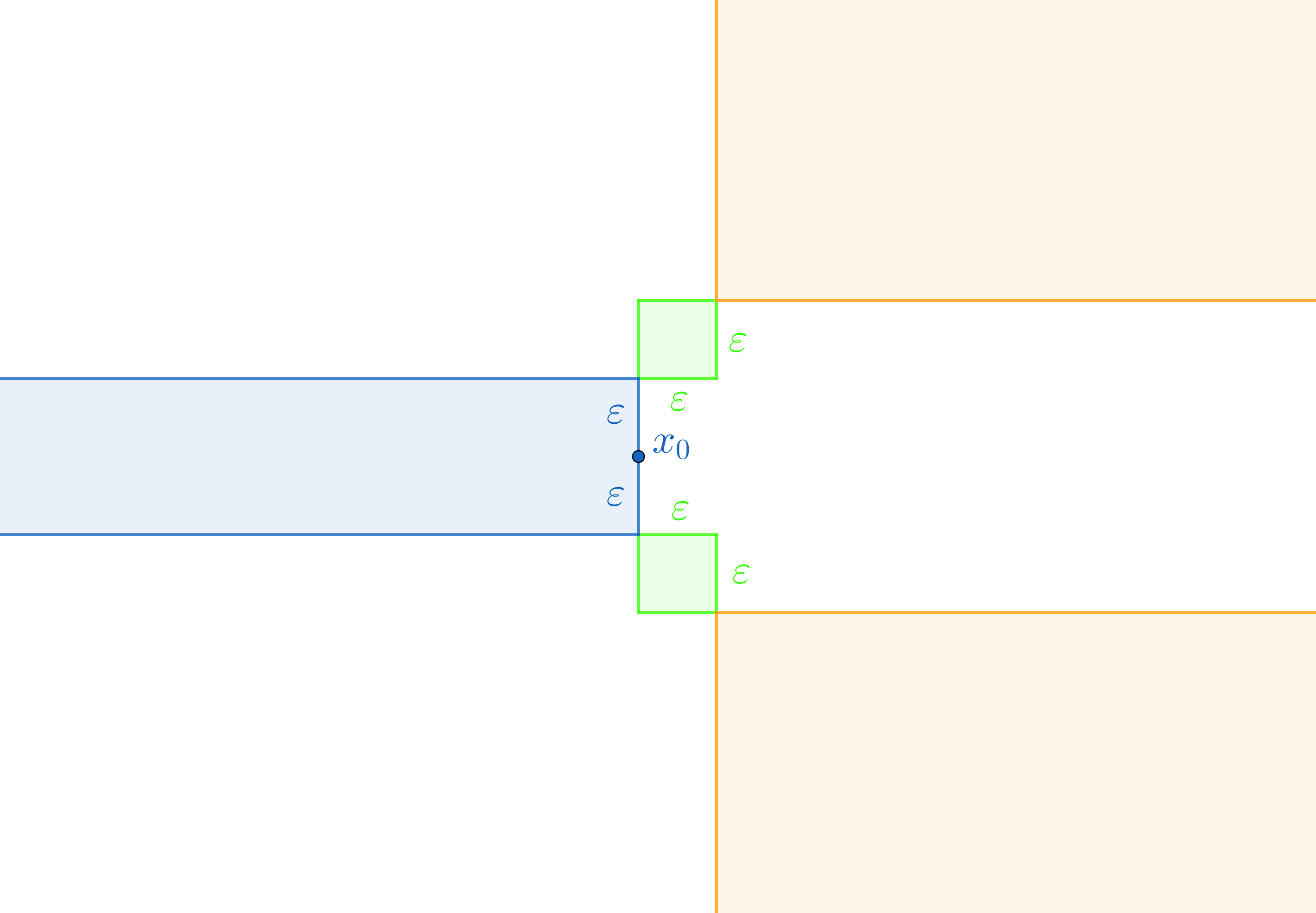}
        	\caption{The points $x_1,...,x_{2^{d-1}}$ are placed in the \textcolor{green}{green boxes} and $x_{2^{d-1}+1},..., x_n$ inside the \textcolor{blue}{blue area}. Because all cuts are used for these points, the values of the tree in the \textcolor{orange}{orange area} must be the same as the ones in the green boxes.}
        	\label{fig-proof-shap-finite-tree}
        \end{figure}
            Although SHAP is informative for $\Tcal_K$, the stronger notion of Rademacher-informativeness, as defined in Section \ref{sec:alternative definition rademacher}, does not apply for $\Tcal_K$, as the function class is not uniformly bounded and $\Fp^{x_0}$ and $\Fe^{x_0}$ hence both have infinite Rademacher complexity. However, the explanation becomes Rademacher-informative (and thus informative) if we restrict the image space to $[-1,1]$, so by dealing with $\Tcal_K^{[-1,1]}$ instead of $\Tcal_K$. In the following we prove both statements, Rademacher-informativeness for $\Tcal_K^{[-1,1]}$ and informativeness for $\Tcal_K$, where the proofs work almost completely analogous.

            \begin{description}
                \item[Rademacher-informativeness for $\Tcal_K^{[-1,1]}$:] Let $\Fcal$ be the class of trees $\Tcal_K$ with values in $[-1,1]$ and depth at most $K$. Let $f\in\Fcal$ and $x_0\in\Xcal$. Note that the expected value of $f$ is given as $\E(f(X)) = f(x_0) - \sum_{i=1}^d \Phi_i(x_0)$ by the efficiency axiom (\ref{eq-efficiency-axiom}). This implies that every $g\in\Fe^{x_0}$ satisfies $\E(g(X))=\E(f(X))$. We are now going to show that if $n$ is large enough, then with positive probability, there occur Rademacher points $x_1,...,x_n$ with labels that can be interpolated by a function $g\in\Fp^{x_0}$ but not by any function in $\Fe^{x_0}$ due to the restriction on the expected value. Denoting the event of these points and labels occurring by $A$, we will then conclude $R_n(\Fe^{x_0}|A)<R_n(\Fp^{x_0}|A)$, which then implies $R_n(\Fe^{x_0})<R_n(\Fp^{x_0})$ by Lemma \ref{lemma-rademacher}.\\
            We assume w.l.o.g. that $\Pr(X^{(1)}\le x_0^{(1)})< 1/2$, in case of ``$>$'' we may simply mirror the proof. As $\Pr(X^{(1)}\le x_0^{(1)})=1/2$ happens with probability zero over the choice of $x_0$, we can ignore this case. Also, we assume w.l.o.g. $\E(f(X))\le 0$, otherwise we can change all Rademacher labels from $1$ to $-1$ and vice versa.\\
            For some $\varepsilon>0$ we consider the boxes $(x_0^{(1)}, x_0^{(1)}+\varepsilon)\times \prod_{j=2}^d B_j$, where $B_j=(x_0^{(j)}+\varepsilon, x_0^{(j)}+2\varepsilon)$ or $B_j=(x_0^{(j)}-2\varepsilon, x_0^{(j)}-\varepsilon)$, leading to $2^{d-1}$ many boxes (see \textcolor{green}{green} parts in Figure \ref{fig-proof-shap-finite-tree}). We consider points $x_1,...,x_{2^{d-1}}$, one in each box and set $\sigma_1=...=\sigma_{2^{d-1}}=1$. Note that these points occur with positive probability. We now consider more points $x_{2^{d-1}+1},...,x_n$ in $(-\infty, x_0^{(1)})\times \prod_{j=2}^d(x_0^{(j)}-\varepsilon, x_0^{(j)}+\varepsilon)$ (\textcolor{blue}{blue} part of Figure \ref{fig-proof-shap-finite-tree}) such that they occur with positive probability and a tree $g\in\Fp^{x_0}$ needs exactly all its $2^K$ leaves to interpolate $x_1,...,x_n$. This can be achieved in the following way: We construct boxes $C_k:=\prod_{j=1}^d (x_0^{(j)}-k\cdot 2^{-n}\varepsilon, x_0^{(j)}-(k-1)\cdot2^{-n}\varepsilon)$ for $k=1,...,2^K-2$ if $\sign(f(x_0))=-1$ and $k=1,...,2^K-1$ if $\sign(f(x_0))=1$, which all lie in the blue part of our figure by construction. Now we consider points $x_{2^{d-1}+1},..., x_n$ such that there is at least one point in every $C_k$ and set $\sigma_i=\sign(f(x_0))$ if $x_i\in C_k$ with $k$ even and $\sigma_i=-\sign(f(x_0))$ if $x_i\in C_k$ with $k$ odd. Like this, the labels of $x_0, C_1,C_2,...$ are alternating. All points $x_1,...,x_n$ can easily be interpolated by splits in the first coordinate. The points $x_1,...,x_{2^{d-1}}$ all have a first coordinate greater than $x_0^{(1)}$, then there comes $x_0$, and then all $C_k$ which have distinct intervals for their first coordinates. At the same time, a tree cannot separate the points with fewer than $2^K$ leaves: If two points in different boxes $C_{k_1}$ and $C_{k_2}$ were classified (including $C_0:=x_0$) within the same leaf, all $C_k$ for $k_1<k<k_2$ would lie in the same leaf and there would hence be some misclassification. So we need one leaf per box and one for $x_0$ and then another one for $x_1,...,x_{2^{d-1}}$ if $f(x_0)=-1$, leading to a total of $2^K$ leaves. Altogether, this means that there is a function $g\in\Fp^{x_0}$ which can assign the points $x_1,...,x_n$ the labels $\sigma_i$ while also needing all its $2^K$ leaves for it. \\
            Note that the features of $x_1,...,x_n$ lie in the interval $(-\infty,x_0^{(1)}+\varepsilon)$ for feature 1 and $(x_0^{(j)}-2\varepsilon, x_0^{(j)}+2\varepsilon)$ for all features $j=2,...,d$. So, this directly implies that there cannot be any splits outside of these intervals for any feature as a function in $\Fp^{x_0}$ needs all its leaves for the points $x_1,...,x_n$. In particular, there are no splits in the set $M_\varepsilon:=[x_0^{(1)}+\varepsilon,\infty)\times \prod_{j=2}^d ((-\infty, x_0^{(j)}-2\varepsilon]\cup [ x_0^{(j)}+2\varepsilon,\infty))$ (\textcolor{orange}{orange} part of Figure \ref{fig-proof-shap-finite-tree}). Hence, for any $g\in\Fp^{x_0}$ interpolating $x_1,...,x_n$, the labels of all points in the set $M_\varepsilon$ are dictated by the ones of $x_1,...,x_{2^{d-1}}$ as they are on the ``outside'' of the set that contains the splits. However, for $\eps$ small enough, the set $M_\eps$ becomes arbitrarily close to $\{X^{(1)}> x_0^{(1)}\}$, which has probability greater than $1/2$. So for $\eps$ small enough, we get $\Pr(M_\eps)=:p>1/2$. But every function $g\in\Fe^{x_0}$ has some non-positive expected value. As the values are bounded from below by $-1$, a non-positive expected value can only be achieved if the value of the function on $M_\eps$ is at most $(1-p)/p$. Hence, a function $g\in\Fe^{x_0}$ either has to misclassify one of the points $x_{2^{d-1}+1},...,x_n$ (in the \textcolor{blue}{blue} area), which would allow an additional split in $M_\eps$, or it has to assign the points $x_1,...,x_{2^{d-1}}$ some value $\le (1-p)/p<1$. Both of these things imply 
            $$\sup_{g\in\Fe^{x_0}} \frac{1}{n} \sum_{i=1}^n \sigma_ig(x_i) 
    	    <\sup_{g\in\Fp^{x_0}} \frac{1}{n} \sum_{i=1}^n \sigma_ig(x_i),$$
            so $R_n(\Fe^{x_0}|A)<R_n(\Fp^{x_0}|A)$, where $A$ is the event of $x_1,...,x_n$ and $\sigma_1,...,\sigma_n$ occurring in the way they did. By Lemma \ref{lemma-rademacher} we conclude $R_n(\Fe^{x_0})<R_n(\Fp^{x_0})$.

            \item[Informativeness for $\Tcal_K$:] We use the same construction for all boxes and sets and placement of the sample $x_1,...,x_n$ as in the case of $\Tcal_K^{[-1,1]}$. What differs is only the assigned values. Let $g\in\Fp^{x_0}$ be the tree with the following properties: On the points $x_1,...,x_{2^{d-1}}$ 
            it holds $g(x_i)=v$, where $v\in\R$. On the points $x_{2^{d-1}+1},...,x_n$ in the boxes $C_k$, the function $g$ assigns the values $(-1)^k\cdot\sign(f(x_0))$. By the same argument as before, a function $\tilde{g}\in\Fp^{x_0}$ needs all its $2^K$ leaves to have the same behavior as $g$ on the sample, meaning it is completely determined. Also, it has to map every point in $M_\eps$ to $v$ and only takes the values $f(x_0)$, 1 and $-1$ away from it. This implies
            $$\E(\tilde{g}(X))\ge v\cdot \Pr(M_\eps) - \max\{1, |f(x_0)|\}.$$
            By choosing $v> (\E(f(X)) + \max\{1,|f(x_0)|\})/\Pr(M_\eps)$, we may conclude $\E(\tilde{g}(X))>\E(f(X))$, so their SHAP values cannot coincide. This means, that $\tilde{g}\notin\Fe^{x_0}$, so no function in $\Fe^{x_0}$ can match $g$ on the sample. As the sample configuration occurs with positive probability, this implies informativeness.
            \end{description}
            
    \end{enumerate}
\end{proof}

\section{Proofs About Anchor Explanations}

\subsection{Proof of Theorem \ref{thm-anchors}}
\label{app:proof-anchors}
\begin{proof}\
    \begin{enumerate}
        \item Let $\Fcal$ be the class of all trees $\Tcal_\infty^{[-1,1]}$ with arbitrary depth and values in $[-1,1]$. Let $f\in\Fcal$, and let $x_0\in\Xcal$. For some $n\in\N$ consider points $x_1,...,x_n\sim \Pr$ and some $g\in\Fp^{x_0}$. With probability 1 over the choice of these points, they are all distinct from $x_0$.  We are going to construct a function $\tilde{g}\in\Fe^{x_0}$ satisfying $\tilde{g}(x_i)=g(x_i)$. For that, we need to make sure that $\tilde{g}$ coincides with $g$ on $x_1,...,x_n$ while the anchor $R$ with precision $p$ is valid for $\tilde{g}$. We write $I_=\subseteq[n]$ for the indices satisfying $c_g(x_i)=c_f(x_0)$ resp. $I_{\neq}\subseteq[n]$ for the indices satisfying $c_g(x_i)\ne c_f(x_0)$. We define small boxes $B_i$ around each $x_i$, where $i\in I_{\ne}$. Each $B_i$ is defined small enough such that it does not contain $x_0$ or any $x_k$, where $k\in I_=$, and such that $\Pr(\Xcal_R\cap\bigcup_{i\in I_{\ne}} B_i)\le (1-p)\cdot \Pr(\Xcal_R)$. Furthermore, we build a set $B\subseteq \Xcal_R$ out of rectangles, which does not contain $x_0$ or any $x_i$, where $i\in I_=$, such that $\Pr(B \cup (\Xcal_R\cap\bigcup_{i\in I_{\ne}} B_i))=(1-p)\cdot \Pr(\Xcal_R)$. Intuitively, the set $B$ is used to ensure that the precision of our anchor is exactly $p$. We now define 
        $$\tilde{g}(x)=\begin{cases}
            g(x_i)\quad &\text{if } x\in B_i\\
            -f(x_0) \quad &\text{if } x\in B\\
            f(x_0) &\text{else}
        \end{cases}.$$ 
        Note that with trees of arbitrary depth, this is possible as they can have arbitrarily many cuts.
        This means that $c_{\tilde{g}}(x)$ differs from $c_f(x_0)$ if and only if $x\in B\cup \bigcup_{i\in I_{\ne}} B_i$. Combined with the equation $\Pr(B \cup (\Xcal_R\cap\bigcup_{i\in I_{\ne}} B_i))=(1-p)\cdot \Pr(\Xcal_R)$ from above, this ensures $\precision(R,g,x_0)=p$, so $\tilde{g}\in\Fe^{x_0}$. We hence constructed a function $\tilde{g}\in\Fe^{x_0}$ matching $g$ on $x_1,...,x_n$, which completes the proof.
        \item We will prove Rademacher-informativeness as defined in Section \ref{sec:alternative definition rademacher}. As shown there, this implies informativeness. Let $\Fcal$ be the class of trees $\Tcal_K^{[-1,1]}$ with values in $[-1,1]$ and depth at most $K$. Let $f\in\Fcal$, and let $x_0\in\Xcal$. The proof works analogously to the proof of the case $\Fcal=\Tcal_K$ in Theorem \ref{thm-shap-trees}. There, every function in $\Fe^{x_0}$ had the same expected value as $f$. We assumed w.l.o.g. $\Pr(X^{(1)}\le x_0^{(1)})< 1/2$ and $\E(f(X))\le 0$. Then, there were Rademacher points constructed such that a tree $g\in\Fp^{x_0}$, which is fitting them, had to yield $\Pr(g(X)=1)>1/2$, which means $\E(g(X))>0$, implying $g\notin \Fe^{x_0}$. \\
        In the case of anchors we can consider the space $\Xcal_R$ instead of $\Xcal$ and assume w.l.o.g. that $\Pr_{\Xcal_R}(X^{(1)}\le x_0^{(1)})\le 1/2$. The restriction coming from the expected value is replaced by the precision that the anchor comes with and assume w.l.o.g. that $p\le 1/2$. If we then construct a Rademacher sample in the same way and choose $\sigma_1=...=\sigma_{2^{d-1}}=c_f(x_0)$, we analogously conclude $\Pr(g(X)=c_f(x_0))>1/2$. This leads to a precision $p>1/2$, again implying $g\notin \Fe^{x_0}$. 
        \item We will prove Rademacher-informativeness as defined in Section \ref{sec:alternative definition rademacher}. As shown there, this implies informativeness. Let $\Fcal$ be the class of trees $\Tcal_\infty^{[-1,1]}$ with values in $[-1,1]$ and arbitrary depth. Let $f\in\Fcal$, and let $x_0\in\Xcal$. Let $n\in\N$. Consider the event $A$ of all Rademacher points $x_1,...,x_n$ falling into $\Xcal_R$, which happens with positive probability, and $\sigma_1=...=\sigma_n=-c_f(x_0)$. As a function $g\in\Fe^{x_0}$ assigns the label $c_f(x_0)$ on $\Xcal_R$ with probability 1, we have \linebreak $\sup_{g\in\Fe^{x_0}} \sum_{i=1}^n\sigma_i g(x_i) <0$. But a tree in $\Fp^{x_0}$ can interpolate all labels with probability 1 as we deal with trees of arbitrary depth. Hence,
        $$R_n(\Fe^{x_0}|A)<R_n(\Fp^{x_0}|A)=1,$$
        which implies $R_n(\Fe^{x_0})<R_n(\Fp^{x_0})$ because of Lemma \ref{lemma-rademacher}.
    \end{enumerate}
\end{proof}

\section{Proofs about counterfactual explanations}

\subsection{Strong counterfactuals are not always well-defined}\label{appendix-strong-counterfactuals}
When introducing strong counterfactuals, we ignore the technical issue that (\ref{eq-counterfactual-strong}) might not have a well-defined solution as $\{x\in\Xcal\condon c_f(x)\ne c_f(x_0)\}$ is not necessarily a closed set. This can be fixed, for example  by considering $c_f(x_0+(1+\eps)v)\ne c_f(x_0)$ for all $\eps>0$ instead. In case the points on the decision boundary get the same label as $x_0$, this will then return a strong counterfactual $x_C=x_0+v$ right on the decision boundary, so with the same label as $x_0$. Although this might seem counterintuitive, we still know that there is no point with a different label inside the ball $B_{\|v\|}(x_0)$, which is the important information we receive from a strong counterfactual. In particular, our proofs work the same way.

\subsection{Proof of Theorem \ref{thm-counterfactuals}}
\label{app:proof-cf}
\begin{proof}\
    \begin{enumerate}
        \item Let $\Fcal$ be the class of arbitrarily deep trees $\Tcal_\infty^{[-1,1]}$, the class of continuous functions $\Ccal^{[-1,1]}$, or the class of differentiable functions $\Dcal^{[-1,1]}$, each with values in $[-1,1]$. Let $f\in\Fcal$, and let $x_0\in\Xcal$. Let $n\in\N$, $x_1,...,x_n\sim \Pr$ be a sample and let $g\in\Fp^{x_0}$. We note that $\Fcal$ is interpolating, so with probability $1$ over the choice of $x_1,...,x_n$, there is a function $\tilde{g}\in\Fcal$ which can fit the points $(x_0, f(x_0)), (x_C  , c_f(x_C)), (x_1,g(x_1)), ..., (x_n,g(x_n))$, where $x_C=\exwcf(f, x_0)$ is the counterfactual. But this means, $\tilde{g}\in\Fe^{x_0}$ and $\tilde{g}$ is matching $g$ on the sample, so the explanation is not informative.
        \item 
            \begin{figure}
            	\centering
            	\includegraphics[width=0.8\linewidth]{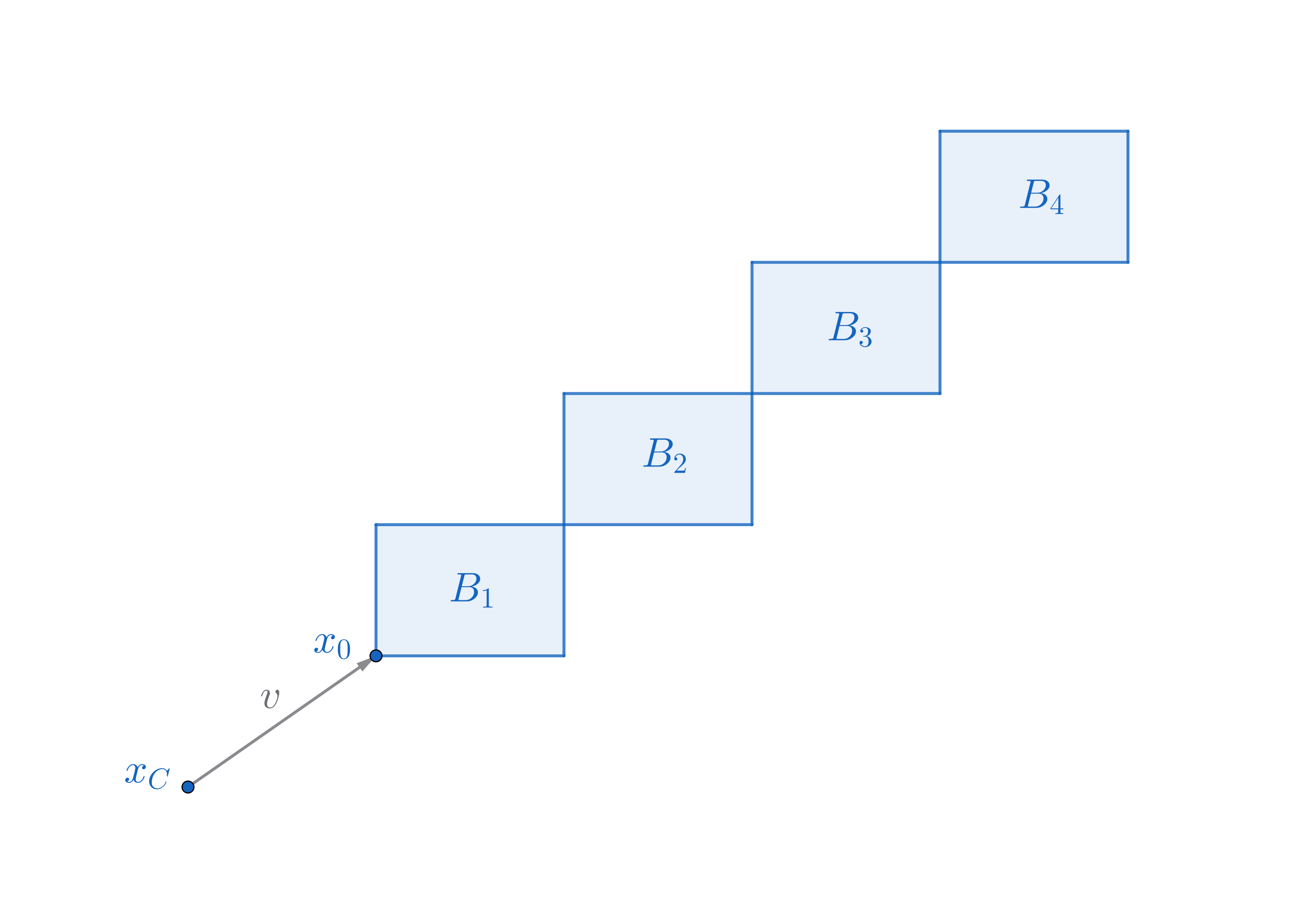}
            	\caption{A visualization of the boxes constructed in the proof of $\Fcal=\Tcal_K$ in Theorem \ref{thm-counterfactuals}.}
            	\label{figure-counterfactuals-finite-trees}
            \end{figure}
        \begin{description}
        	\item[$\Fcal=\Tcal_K^{[-1,1]}$:] We will prove Rademacher-informativeness as defined in Section \ref{sec:alternative definition rademacher}. As shown there, this implies informativeness. Let $f\in\Fcal$, and $x_0\in\Xcal$. We will construct a sample of Rademacher points with binary labels that can be interpolated by a function in $\Fp^{x_0}$ but not by any function in $\Fe^{x_0}$. If we denote the event of these points and labels occurring by $A$, we can conclude $R_n(\Fe^{x_0}|A)<R_n(\Fp^{x_0}|A)$, which then implies $R_n(\Fe^{x_0})<R_n(\Fp^{x_0})$ by Lemma \ref{lemma-rademacher}.\\
        	We define $\tilde{v}:=x_0-x_C$. If $\tilde{v}$ has a zero in some of its components, we replace them by $1$ and call the new vector we receive $v$. We then define boxes $B_k:=(x_0^{(1)}+(k-1)v^{(1)}, x_0^{(1)}+kv^{(1)})\times...\times(x_0^{(d)}+(k-1)v^{(d)}, x_0^{(d)}+kv^{(d)})$ for $k=1,...,2^K-1$ (see Figure~\ref{figure-counterfactuals-finite-trees}). Now consider points $x_1,...,x_n\in\bigcup_{k=1}^{2^K-1}B_k$, such that there is at least one point in each box. Note that this happens with positive probability if $n\ge 2^K-1$ as every entry of $v$ is non-zero and all $B_k$ hence have positive probability. We set $\sigma_i=c_f(x_0)$ if $x_i$ lies in a box $B_k$ with $k$ even and we set $\sigma_i=-c_f(x_0)$ if $x_i$ lies in a box $B_k$ with $k$ odd. Like that, the labels of $x_0, B_1,B_2,...,B_{2^K-1}$ are alternating. Then, a tree with depth at most $K$ can interpolate $x_0,...,x_n$ just by using the first coordinate, for example. In particular, there is a function in $\Fp^{x_0}$, which can interpolate $x_1,...,x_n$. However, it needs all its $2^K$ leaves: If two points from distinct boxes $B_{k_1}$ and $B_{k_2}$ with the same label were classified within the same leaf, all boxes $B_k$ for $k_1<k<k_2$ would also lie within this leaf, so there would be at least one point misclassified. But if we now add the counterfactual to the points, using the same argument, a tree would need at least $2^K+1$ many leaves to interpolate $x_C,x_0,x_1,...,x_n$. This implies that a tree $g$ in $\Fe^{x_0}$ misclassifies at least one of the points $x_1,...,x_n$. As we then have $|g(x_i)- \sigma_i|>1$ for this point, we know
            $$\frac{1}{n}\sup_{g\in\Fe^{x_0}}\sum_{i=1}^n \sigma_ig(x_i)\le 1-1/n,$$
            whereas  $$\frac{1}{n}\sup_{g\in\Fp^{x_0}}\sum_{i=1}^n \sigma_ig(x_i)=1.$$
            But this now implies $R_n(\Fe^{x_0}|A) < R_n(\Fp^{x_0}|A)$ and it follows $R_n(\Fe^{x_0})< R_n(\Fp^{x_0})$ by Lemma \ref{lemma-rademacher}.
        	
        	\item[$\Fcal=\Ccal_L$:] We will prove Rademacher-informativeness as defined in Section \ref{sec:alternative definition rademacher}. As shown there, this implies informativeness. Let $f\in\Fcal$, let $x_0\in\Xcal$, and let $n\in\mathbb{N}$. Due to the Lipschitz-continuity, we know that for all $g\in\Fcal$ and all points $x\in\Xcal$ with $\|x-x_C\|<g(x_C)/L$ we have $\sign(g(x))=\sign(g(x_C))$. Now consider the event $A$ that all Rademacher points $x_1,...,x_n$ fall into $\{x\in\Xcal\condon \|x-x_C\|<|f(x_C)|/L\}$, which happens with positive probability, and have labels $\sigma_1=...=\sigma_n=-\sign(f(x_C))=-c_f(x_C)$. Every function  $g\in\Fe^{x_0}$ then satisfies $\sign(g(x_i)) = \sign(g(x_C)) = \sign(f(x_C))$ and hence $\sum_{i=1}^n\sigma_ig(x_i)<0$. However, the constant function $h\equiv f(x_0)$  is contained in $\Fp^{x_0}$ and satisfies $\frac{1}{n}\sum_{i=1}^n\sigma_ih(x_i)=\frac{1}{n}\sum_{i=1}^n \sign(f(x_0))\cdot f(x_0)=|f(x_0)|>0.$ Hence,
        	\ba
        	\sup_{g\in\Fe^{x_0}} \sum_{i=1}^n \sigma_ig(x_i)\le 0
        	<|f(x_0)|
        	&\le\sup_{g\in\Fp^{x_0}} \sum_{i=1}^n \sigma_ig(x_i),
        	\ea
            so $R_n(\Fe^{x_0}|A)< R_n(\Fp^{x_0}|A)$         	and it follows
        	$R_n(\Fe^{x_0})<R_n(\Fp^{x_0})$
        	by Lemma \ref{lemma-rademacher}.
        \end{description}

        \item We will prove Rademacher-informativeness as defined in Section \ref{sec:alternative definition rademacher}. As shown there, this implies informativeness. Let $\Fcal$ be the class of arbitrarily deep trees $\Tcal_\infty^{[-1,1]}$, the class of continuous functions $\Ccal^{[-1,1]}$, or the class of differentiable functions $\Dcal^{[-1,1]}$, each with values in $[-1,1]$. Let $f\in\Fcal$, let $x_0\in\Xcal$, and let $n\in\mathbb{N}$. Let us assume that at a point $x_0\in\mathbb{R}^d$ we receive a prediction $f(x_0)$ and a strong counterfactual explanation $x_C$ and the corresponding vector $v=x_C-x_0\in\mathbb{R}^d$. Because $v$ is a minimal step (as it is a solution of the optimization problem (\ref{eq-counterfactual-strong})), it holds that $\forall x\in B_{\|v\|}(x_0):c_f(x)=c_f(x_0)$. Let $A$ be the event in which all samples fall into $B_{\|v\|}(x_0)$ and $\forall i\in[n]:\sigma_i=-c_f(x_0)$. This event has a positive probability. By Lemma \ref{lemma-rademacher}, it is enough to show that $R_n(\Fe^{x_0}|A)<R_n(\Fp^{x_0}|A)$. It is easy to see that this is the case, because $R_n(\Fe^{x_0}|A)<0$ always holds and $R_n(\Fp^{x_0}|A)\geq0$ can always be achieved, since functions in $\Fcal$ can interpolate any finite sample of points.
    \end{enumerate}
\end{proof}

\subsection{Alternative definitions of counterfactuals}
\label{subsec-alternative-cf-def}
In practice, it is hardly the case that a counterfactual is guaranteed to be the closest possible point of a different label, as search algorithms use heuristics. Our definition of a weak counterfactual reflects this idea. However, there are also other ways of defining a counterfactual that still contain some notion of proximity in their definition. One example is given by counterfactuals along coordinate axes (as often used in practice). For those, we solve the (easier) optimization problem

\begin{equation} \label{eq-counterfactual-one-dim}
\begin{aligned}
    \argmin_{j\in[d],\,r\in\mathbb{R}}\quad & |r|\\
    \text{s.t.} \quad & c_f(x_0)\neq c_f(x_0+r\cdot e_j),
\end{aligned}
\end{equation}

where $e_j$ is the $j$-th unit vector. The regarding counterfactual is then given by $x_C:=x_0+ r^*e_{j^*}$, where $j^*$ and $r^*$ are solutions of (\ref{eq-counterfactual-one-dim}). 

However, it turns out that this definition also leads to an uninformative explanation in the case of complex function classes. The key to this is that only a zero set of points is considered when searching for the minimum solely along the coordinate axes. For that reason, given a sample of points and a function $g\in\Fp^{x_0}$, it is easy to find a function in $\Fe^{x_0}$, which behaves on the coordinate axes in such a way that the counterfactual is valid while admitting the same prediction behavior as $g$ on the sample points. 

Generally, all relaxed variants of the classical (strong) definition of counterfactuals that we have worked with ended up leading to an uninformative explanation in the case of complex function classes. As soon as the guarantee of dealing with the closest possible point is not given, the counterfactual can no longer guarantee a restriction of the function class. 

\section{Proof of Theorem \ref{thm-multiple-explanations} (several explanations)}
\label{app:proof-several-explanations}
\begin{proof}
    All of our negative results carry over to the case of several points and explanations. In the following, we describe how the proofs need to be adapted. Just like in the case of one explanation, we use $x_1,...,x_n$ for the Rademacher points whereas we write $z_1,...,z_m$ for points for which the explanations are given.
    \begin{description}
        \item[Theorem \ref{thm-gradients}, part 1:] The construction used for $\Fcal=\Dcal$ can be done for several points by defining the neighborhood around every point such that they are disjoint. This is almost surely possible as the points $z_1,...,z_m$ are almost surely different.\\
        Also for $\Fcal=\Dcal_\alpha$, we use this strategy: We set the value as $r:=\tfrac{1}{2}\min_{\substack{p,q\in\{x_1,...,x_n,z_1,...,z_m\}\\ p\neq q}} \|p-q\|$ to make sure the balls around the $z_j$ do not intersect. Then, for every point $z_j$, we define an individual $\phi_j$ and $k_j$ and the rest of the proofs goes analogously because the supports of the $\phi_j$ do not intersect.\\
        The argument in the case $\Fcal=\Pcal$ does not change in the case of several points and explanations, as our polynomials are not restricted in degree.
        
        \item[Theorem \ref{thm-shap-trees}, part 1:]  Note that the proof relies on the fact that the SHAP value of a point $z_j$ only depends on all points $D_j$ that share at least one coordinate with $z_j$ (which is a zero set) and on $\E(f(X))$. If we consider several points $z_1,...,z_m$, the union $D:=\bigcup_{j=1,...,m}D_j$ is still a zero set. The proof then works analogously by working with $D$ defined as above.

        \item[Theorem \ref{thm-shap-trees}, part 2:] Here, we simply define more intervals $I_1,...,I_{m+n}$ around $z_1,...,z_m,x_1,...,x_n$, which have probabilities $p_1,...,p_{m+n}$. The rest of the proof then stays the same if we let the indices run from 1 to $m+n$ instead of 0 to $n$. As the expected values of $f_1$ and $\tilde{g}_1$ coincide, so do the SHAP values in all $z_j$.

        \item[Theorem \ref{thm-anchors}, part 1:]  Here, we get several anchors $(R_1, p_1),...,(R_m,p_m)$ around $z_1,...,z_m$ as well as separate sets of indices $I^j_=$ and $I^j_{\ne}$ for every $z_j$ (yielding $c_g(x_i)=c_f(z_j)$ resp. $c_g(x_i)\ne c_f(z_j)$). When defining the boxes $B_1,...,B_n$, one needs to define them small enough such that the inequality $\Pr(\Xcal_{R_j}\cap\bigcup_{i\in I^j_{\ne}} B_i)\le (1-p_j)\cdot \Pr(\Xcal_{R_j})$ holds \emph{for all} $j=1,...,m$. Instead of the set $B$, we then define separate sets $C_1,...,C_m$ within $\Xcal_{R_1},...,\Xcal_{R_m}$, such that $\Pr(C_j \cup (\Xcal_{R_j}\cap\bigcup_{i\in I^j_{\ne}} B_i))=(1-p_j)\cdot \Pr(\Xcal_{R_j})$ holds for all $j=1,...,m$. Finally, we may define
        $$\tilde{g}(x)=\begin{cases}
            g(x_i)\quad &\text{if } x\in B_i\\
            -f(z_j) \quad &\text{if } x\in C_j\\
            f(z_j) &\text{if } x\in\Xcal_{R_j}\setminus (C_j\cup \bigcup_{i=1}^nB_i)\\
            0 &\text{else}
        \end{cases}.$$ 
        to achieve $\tilde{g}(x_i)=g(x_i)$ as well as the desired precisions.

        Note that the theorem no longer needs to hold if the anchors are allowed to overlap. Combining two imperfect, overlapping anchors might yield a perfect anchor, leading to informativeness.
        \item[Theorem \ref{thm-counterfactuals}, part 1:] This proof remains exactly the same; we just interpolate several points and counterfactuals together with our Rademacher points, rather than just one. 
    \end{description}
\end{proof}

\end{document}